\pdfoutput=1
\documentclass[11pt]{article}
\usepackage{fullpage}
\usepackage{mathtools,amsfonts,amsthm,amssymb,bbm}
\usepackage{xcolor}
\usepackage[backref,colorlinks,citecolor=blue,linkcolor=magenta,bookmarks=true]{hyperref}
\usepackage[nameinlink]{cleveref}
\usepackage[algo2e,ruled,linesnumbered]{algorithm2e}

\title{On the Distance from Calibration in Sequential Prediction}
\author{
Mingda Qiao\thanks{University of California, Berkeley. Email: \texttt{mingda.qiao@berkeley.edu}. Part of this work was done while the author was a graduate student at Stanford University.}
\and
Letian Zheng\thanks{University of California, Los Angeles. Email: \texttt{letianzh@g.ucla.edu}.}
}
\date{}

\newcommand{\1}[1]{\mathbbm{1}\left[#1\right]} 
\newcommand{\A}{\mathcal{A}} 

\newcommand{\bbN}{\mathbb{N}}
\newcommand{\Bern}{\mathsf{Bernoulli}} 
\newcommand{\Binomial}{\mathsf{Binomial}} 
\newcommand{\Cal}{\mathcal{C}}
\newcommand{\caldist}{\mathsf{CalDist}}
\newcommand{\cost}{\mathsf{Cost}}
\newcommand{\sfl}{\mathsf{L}}
\newcommand{\sfr}{\mathsf{R}}
\newcommand{\sfo}{\mathsf{out}}
\newcommand{\sfi}{\mathsf{in}}
\newcommand{\costl}{\cost^{\sfl}}
\newcommand{\costr}{\cost^{\sfr}}
\newcommand{\costlo}{\cost^{\sfl,\sfo}}
\newcommand{\costli}{\cost^{\sfl,\sfi}}
\newcommand{\costro}{\cost^{\sfr,\sfo}}
\newcommand{\costri}{\cost^{\sfr,\sfi}}

\newcommand{\D}{\mathcal{D}} 
\newcommand{\Deltal}{\Delta^{\sfl}}
\newcommand{\Deltar}{\Delta^{\sfr}}
\newcommand{\ECE}{\mathsf{ECE}}
\newcommand{\eps}{\epsilon}
\newcommand{\event}{\mathcal{E}}
\newcommand{\Ex}[2]{\operatorname*{\mathbb{E}}_{#1}\left[#2\right]} 
\newcommand{\F}{\mathcal{F}}

\newcommand{\Int}{\mathcal{I}}
\newcommand{\len}{\mathsf{len}}
\newcommand{\lowerCal}{\underline{\mathcal{C}}}
\newcommand{\lowercaldist}{\mathsf{LowerCalDist}}
\newcommand{\N}{\mathcal{N}} 
\newcommand{\poly}{\mathrm{poly}}
\newcommand{\polylog}{\mathrm{polylog}}
\newcommand{\pr}[2]{\Pr_{#1}\left[#2\right]} 
\newcommand{\R}{\mathbb{R}} 
\newcommand{\rmd}{\mathrm{d}}
\newcommand{\sgn}{\mathrm{sgn}} 
\newcommand{\smCE}{\mathsf{smCE}}
\newcommand{\SRC}{\mathcal{SR}} 
\newcommand{\T}{\mathcal{T}}
\newcommand{\Tmixed}{\T^{\mathsf{mixed}}}
\newcommand{\Tpure}{\T^{\mathsf{pure}}}
\newcommand{\unitl}{\mathsf{unit}^{\sfl}}
\newcommand{\unitr}{\mathsf{unit}^{\sfr}}

\newtheorem{theorem}{Theorem}
\newtheorem{definition}{Definition}
\newtheorem{question}{Question}
\newtheorem{lemma}{Lemma}
\newtheorem{proposition}[lemma]{Proposition}
\newtheorem{corollary}[lemma]{Corollary}
\newtheorem{remark}[lemma]{Remark}

\begin{document}

\maketitle

\begin{abstract}%
    We study a sequential binary prediction setting where the forecaster is evaluated in terms of the \emph{calibration distance}, which is defined as the $L_1$ distance between the predicted values and the set of predictions that are perfectly calibrated in hindsight. This is analogous to a calibration measure recently proposed by B\l{}asiok, Gopalan, Hu and Nakkiran (STOC 2023) for the offline setting. The calibration distance is a natural and intuitive measure of deviation from perfect calibration, and satisfies a Lipschitz continuity property which does not hold for many popular calibration measures, such as the $L_1$ calibration error and its variants.

    We prove that there is a forecasting algorithm that achieves an $O(\sqrt{T})$ calibration distance in expectation on an adversarially chosen sequence of $T$ binary outcomes. At the core of this upper bound is a structural result showing that the calibration distance is accurately approximated by the \emph{lower calibration distance}, which is a continuous relaxation of the former. We then show that an $O(\sqrt{T})$ lower calibration distance can be achieved via a simple minimax argument and a reduction to online learning on a Lipschitz class.

    On the lower bound side, an $\Omega(T^{1/3})$ calibration distance is shown to be unavoidable, even when the adversary outputs a sequence of independent random bits, and has an additional ability to \emph{early stop} (i.e., to stop producing random bits and output the same bit in the remaining steps). Interestingly, without this early stopping, the forecaster can achieve a much smaller calibration distance of $\polylog(T)$.
\end{abstract}

\section{Introduction}
We revisit the sequential binary prediction setup of Foster and Vohra~\cite{FV98}, in which a forecaster makes probabilistic predictions on a sequence of $T$ adversarially chosen binary outcomes. At each step $t \in [T]$, the adversary picks a bit $x_t \in \{0, 1\}$ and, simultaneously, the forecaster makes a prediction $p_t \in [0, 1]$ on the ``probability'' of $x_t = 1$. These values are then revealed to both players, and may factor into their subsequent actions.

The forecaster is evaluated in terms of the \emph{calibration} criterion, which is a natural and intuitive condition for the predictions to be interpretable as probabilities. The predictions are called \emph{perfectly calibrated} if, among the steps on which each $\alpha \in [0, 1]$ is predicted, exactly an $\alpha$ fraction of the bits are ones. Formally, it must hold for every $\alpha \in [0, 1]$ that $\sum_{t=1}^{T}(x_t - p_t)\cdot \1{p_t = \alpha} = 0$.

Quantitatively, the $L_1$ calibration error, also known as the Expected Calibration Error (ECE), is defined as the total violation of calibration over all $\alpha \in [0, 1]$\footnote{The summand is non-zero only if $\alpha \in \{p_1, p_2, \ldots, p_T\}$, so the summation is finite and well-defined. In the rest of the paper, we frequently abuse the notation $\sum_{\alpha \in [l, r]}x_{\alpha}$ if the $x_\alpha$ is non-zero on finitely many values of $\alpha \in [l, r]$.}:
\begin{equation}\label{eq:ECE}
    \ECE(x, p) \coloneqq \sum_{\alpha \in [0, 1]}\left|\sum_{t=1}^{T}(x_t - p_t)\cdot \1{p_t = \alpha}\right|.
\end{equation}
While this definition seems natural, the ECE can be \emph{ill-behaved} because it is discontinuous with respect to the predictions. For instance, when $x$ contains the same number of $0$s and $1$s, predicting $p_t = 1/2$ at every $t \in [T]$ achieves a zero ECE. However, if we replace each $p_t$ with $1/2 + \eps_t$, where $\eps_1, \eps_2, \ldots, \eps_T$ are arbitrarily small, non-zero, and distinct perturbations, the ECE suddenly jumps to $\Omega(T)$, as the $T$ steps count as $T$ different ``bins'' in Equation~\eqref{eq:ECE}. As noted by~\cite{BGHN23}, while this issue can be alleviated by binning the prediction values, the binning may introduce discontinuity at the boundary of each sub-interval, and there is no consensus on how the binning should be chosen in general.

In this work, we study a variant of this fundamental sequential calibration setup, in which the distance from calibration is defined as the minimum $L_1$ distance between the predicted values and the closest predictions that are perfectly calibrated with respect to the outcomes. This definition is analogous to a calibration measure recently proposed and studied by~\cite{BGHN23} for the offline setting. Formally, with respect to outcomes $x_1, x_2, \ldots, x_T$, the \emph{calibration distance} of predictions $p_1, p_2, \ldots, p_T$ is defined as
\[
    \caldist(x, p) \coloneqq \min_{q \in \Cal(x)}\|p - q\|_1,
\]
where $\Cal(x) \coloneqq \left\{p \in [0, 1]^T: \forall \alpha\in[0,1],~\sum_{t=1}^{T}(x_t - p_t)\cdot\1{p_t = \alpha} = 0\right\}$ is the set of predictions that are perfectly calibrated for $x$.\footnote{Note that every $p \in \Cal(x)$ corresponds to a partition of $[T]$, so $\Cal(x)$ is a finite set of size $T^{O(T)}$. Therefore, the minimum in the definition of $\caldist$ can be achieved.}

Equivalently, the calibration distance measures the minimum amount of modification that the forecaster has to make to its predictions, in order for them to be calibrated in hindsight. It follows immediately from the definition that the calibration distance is robust to small perturbations in the predictions, and thus avoids the discontinuity issue of the ECE. In Appendix~\ref{sec:preliminaries-proofs}, we prove that the calibration distance is always upper bounded by the ECE, so the calibration distance can also be viewed as a relaxation of the ECE.

In this work, we address the following two questions regarding this new calibration measure.
\begin{question}\label{question:computation}
    Can we efficiently compute (or at least approximate) the calibration distance on given outcomes and predictions?
\end{question}

\begin{question}\label{question:bounds}
    What is the optimal calibration distance that the forecaster can guarantee against $T$ adversarially chosen outcomes?
\end{question}

\subsection{Overview of Our Results}
\paragraph{Efficient approximation via a structural result.} We start by giving a positive answer to Question~\ref{question:computation}, up to a small additive approximation error.
\begin{theorem}\label{thm:approx}
    There is an algorithm that, given $x \in \{0, 1\}^T$ and $p \in [0, 1]^{T}$, outputs an estimate of $\caldist(x, p)$ up to an $O(\sqrt{T})$ additive error in $\poly(T)$ time.
\end{theorem}

We prove Theorem~\ref{thm:approx} by relating $\caldist(x, p)$ to the \emph{lower calibration distance}, denoted by $\lowercaldist(x, p)$, which we formally define in Section~\ref{sec:preliminaries}. Roughly speaking, the definition of $\lowercaldist(x, p)$ allows us to compare $p_1, p_2, \ldots, p_T$ to \emph{randomized predictions} $q_1, q_2, \ldots, q_T$. For the offline setting, \cite{BGHN23} introduced an analogous notion, and gave a $\poly(T, 1/\eps)$ time algorithm that approximates the lower calibration distance up to any additive error of $\eps > 0$. Therefore, Theorem~\ref{thm:approx} would immediately follow if we could show that $\lowercaldist(x, p)$ is a good approximation of $\caldist(x, p)$ for any $x$ and $p$.

A result of~\cite{BGHN23} implies that, after normalizing by a $1/T$ factor, these two measures are indeed polynomially related:
\[\frac{\lowercaldist(x, p)}{T} \le \frac{\caldist(x, p)}{T} \le 4\sqrt{\frac{\lowercaldist(x, p)}{T}}.\]
On the other hand, this quadratic gap is unavoidable in general: We show in Proposition~\ref{prop:multiplicative-approx} that even for $T = 4$, there exist $x$ and $p$ such that $\caldist(x, p) = \Omega(\eps)$ but $\lowercaldist(x, p) = O(\eps^2)$ for sufficiently small $\eps > 0$. Taking $\eps \to 0^{+}$ shows that $\lowercaldist(x, p)$ is not a good \emph{multiplicative} approximation of $\caldist(x, p)$.

Fortunately, in the example above, the \emph{additive} gap between the two calibration measures is small. Our key technical result below states that this is true in general: $\caldist(x, p)$ is always upper bounded by $\lowercaldist(x, p) + O(\sqrt{T})$. Thus, Theorem~\ref{thm:approx} indeed follows from the algorithm of~\cite{BGHN23} for approximating $\lowercaldist(x, p)$. Furthermore, when $p$ is \emph{sparse} in the sense that it contains only a few different entries, we improve the additive gap from $\sqrt{T}$ to the sparsity level, at the cost of an extra constant factor.

\begin{theorem}\label{thm:lowercaldist-vs-caldist}
    For any $x \in \{0, 1\}^T$ and $p \in [0, 1]^T$, we have:
    \begin{itemize}
        \item $\caldist(x, p) \le \lowercaldist(x, p) + O(\sqrt{T})$.
        \item $\caldist(x, p) \le O(1)\cdot \lowercaldist(x, p) + O(m)$, where $m = |\{p_1, p_2, \ldots, p_T\}|$.
    \end{itemize}
    The $O(\cdot)$ notations above hide universal constant factors that are independent of $T$, $x$, and $p$.
\end{theorem}

\paragraph{Upper bound via a minimax argument.} Our next result addresses Question~\ref{question:bounds} from the upper bound side.
\begin{theorem}\label{thm:upper}
    There is forecasting algorithm that, against any adversary, achieves an $O(\sqrt{T})$ calibration distance in expectation.
\end{theorem}

In light of Theorem~\ref{thm:lowercaldist-vs-caldist}, it suffices to give a forecaster with an expected lower calibration error of $O(\sqrt{T})$. \cite{BGHN23} showed that the lower calibration distance and the \emph{smooth calibration error} (which we define in Section~\ref{sec:preliminaries}) differ by a constant factor. Therefore, at the core of our proof of Theorem~\ref{thm:upper} is an $O(\sqrt{T})$ upper bound on the smooth calibration error, which is proved via a minimax argument similar to the proof of Hart for upper bounding the ECE~\cite{FV98,Hart22}. We note that Kakade and Foster~\cite{KF08} gave an algorithm with a sub-linear smooth calibration error, though directly following their proof gives a looser upper bound of $O(T^{2/3})$.

\paragraph{Impossibility of impossibility results from random bits.}
It might appear ``obvious'' that the $O(\sqrt{T})$ bound in Theorem~\ref{thm:upper} is tight: Suppose that the adversary plays a sequence of $T$ independent random bits. Intuitively, the forecaster's best strategy is to predict $p_t = 1/2$ at every step $t$. Then, the calibration distance can be shown to be $\Omega(\sqrt{T})$ in expectation.\footnote{This follows from $\caldist(x, p) \ge \lowercaldist(x, p) \ge \frac{1}{2}\smCE(x, p) \ge \frac{1}{2}\left|\sum_{t=1}^{T}(x_t - p_t)\right|$ and that the last term is $\Omega(\sqrt{T})$ in expectation; see Section~\ref{sec:preliminaries} for justification of the first three steps.}

Surprisingly, this argument turns out to be incorrect---in fact, ``exponentially'' incorrect!\footnote{Nevertheless, this argument indeed shows that the proof strategy of Theorem~\ref{thm:upper} at best gives an $O(\sqrt{T})$ bound.}

\begin{proposition}\label{prop:random-bits-polylog}
    When the adversary promises to play $T$ independent random bits, there is a forecasting algorithm that achieves an $O(\log^{3/2}T)$ calibration distance in expectation.
\end{proposition}

Our proof of Proposition~\ref{prop:random-bits-polylog} has two steps: First, we give a strategy that achieves a small smooth calibration error. This is done by first predicting $1/2$ and then, based on the realization of the random bits, predicts a slightly biased value in the hope of de-biasing the previous mistakes. To translate this to an upper bound on the calibration distance, the first bound in Theorem~\ref{thm:lowercaldist-vs-caldist} is insufficient, since the $\sqrt{T}$ gap would dominate the $\polylog(T)$ error. Fortunately, in the first step, we always predict at most $O(\log T)$ different values, so the second bound in Theorem~\ref{thm:lowercaldist-vs-caldist} can be applied instead. 

\paragraph{Lower bounds from random bits.} Despite the surprising fact above, we still manage to prove a $\poly(T)$ lower bound.
\begin{theorem}\label{thm:lower}
    There is a strategy for the adversary such that any forecasting algorithm must incur an $\Omega(T^{1/3})$ calibration distance in expectation.
\end{theorem}

Theorem~\ref{thm:lower} is proved by providing a minimal additional ability to the adversary that produces random bits. The new adversary will generate random bits until the calibration distance hits $\Omega(T^{1/3})$ at some point, and then either keep playing zeros or keep playing ones, depending on which bit could ensure that the calibration distance is still large in the end.

\subsection{Related Work}\label{sec:related-work}
Calibration is a natural criterion for evaluating probabilistic forecasts. The idea of calibration can be at least traced back to Brier~\cite{Brier50}. A formal definition of calibration appeared in the work of Dawid~\cite{Dawid82,Dawid85}. There are huge bodies of recent work on the calibration of neural networks~\cite{GPSW17} and the use of calibration (and its extension such as multi-calibration) as a measure of algorithmic fairness~\cite{KMR17,HKRR18}. In the following, we focus the discussion mainly on calibration in sequential setups, which is the closest to this paper.

\paragraph{Distance from calibration.} In the context of offline probabilistic prediction, \cite{BGHN23} noted that while ``the notion of perfect calibration is well-understood'', ``there is no consensus'' on how the distance from perfect calibration should be quantified. They proposed to use the following as the ground truth for the distance of a predictor from calibration: the minimum $\ell_1$ distance between the predictor and any predictor that is perfectly calibrated with respect to the underlying distribution. The authors then examined various calibration measures, and identified which of them are \emph{consistent} in the sense that of being polynomially related to this ground truth.

In the sequential setup, however, even the notion of ``perfect calibration'' might be at odds with what intuitively count as the ``right'' predictions. It is easy to construct examples in which the adversary generates the outcomes randomly from a known distribution, yet the only way to achieve perfect calibration is through ``lying'' on some predictions, i.e., predicting a value that is far from the true (conditional) probability of the next outcome (see, e.g., \cite[Example 2]{QV21} for a concrete instance).

On a technical level, the notion of consistency in the study of~\cite{BGHN23} might be loose by a quadratic factor (which is also shown to be unavoidable in their formulation). In contrast, the goal of this work is to pin down the optimal rate of the calibration distance, so this quadratic gap poses a challenge. Roughly speaking, the quadratic gap that is unavoidable in the setup of~\cite{BGHN23} is due to the uncertainty in the granularity of the marginal distribution. This gap is partially avoided in the sequential setup, since the ``marginal'' is always the uniform distribution over $[T]$.

\paragraph{Sequential calibration and variants.} Foster and Vohra~\cite{FV98} gave the first algorithm that achieves a vanishing (squared $L_2$) calibration error as $T \to +\infty$ on a sequence of $T$ adversarially generated outcomes. Alternative proofs were subsequently given by \cite{FL99,Foster99}. In terms of the error rates, the optimal ECE (defined in Equation~\eqref{eq:ECE}) is known to be between $O(T^{2/3})$ (implict in~\cite{FV98}; see~\cite{Hart22} for a formal exposition) and $\Omega(T^{0.528})$~\cite{QV21}.

\cite{FRST11}~studied a strengthened notion of calibration that requires the predictions to be calibrated even when restricted to certain subsets of the time horizon (also called ``checking rules''). They derived convergence bounds that depend on different complexity measures of the family of checking rules.

Relaxed versions of the ECE, including weak calibration~\cite{KF08}, smooth calibration~\cite{FH18} and continuous calibration~\cite{FH21},  have also been studied. These alternative calibration notions also resolve the discontinuity issue of ECE, and were shown to be achievable by deterministic forecasting algorithms. In contrast, any algorithm with a sub-linear ECE must be randomized. \cite{GR22} studied a ``power of two choice'' variant in which the forecaster is allowed to predict two different (yet nearby) values at each step, and use the one closer to the outcome after the outcome is revealed.

\paragraph{Calibration-accuracy trade-off.} Another variant of the problem is when the forecaster is given a hint or expert advice before predicting at each step. The goal is to re-calibrate the expert's predictions without increasing the cumulative loss in the predictions. \cite{KE17,OKS23} gave trade-offs between the ECE incurred by the forecaster and the \emph{regret}, defined as the excess loss compared to always following the hints.

\paragraph{Online multi-calibration.} Multi-calibration,  proposed by~\cite{HKRR18}, is a stronger notion of calibration in the context of fairness of machine learning models in offline setups. This notion requires the predictions to be calibrated on a family of pre-specified subsets of the feature space as well. A recent line of work~\cite{GJNPR22,LNPR22,BGJNRR22,GJRR24} gave algorithms that achieve approximate multi-calibration in the online setup, in which the features and labels are sequentially and adversarially chosen.

\paragraph{Calibrated predictions for decision-making.} Recent work of~\cite{KLST23} and \cite{NRRX23} studied models in which the predictions are used for downstream decision-making. \cite{KLST23} defined ``U-Calibration'', which is shown to be equivalent to the sub-linear regret guarantee for \emph{all} decision makers. \cite{NRRX23} gave algorithms that are calibrated even when evaluated in conjunction with the decisions, which might, in turn, depend on the predictions.

\subsection{Organization of the Paper}
In Section~\ref{sec:preliminaries}, we formally introduce two calibration measures in the literature, the \emph{lower calibration distance} and the \emph{smooth calibration error}, both of which are closely related to the calibration distance, and will play key roles in our proofs.
In Section~\ref{sec:proof-overview}, we sketch the proofs of our main results. We suggest that the readers read this section before delving into the formal proofs, as most of the proofs are based on simple ideas and intuition that need slightly heavier notations to be formalized. We 
discuss the suitability of the calibration distance as a calibration measure, and highlight a few open problems
in Section~\ref{sec:discussion}. The formal proofs are given in Sections \ref{sec:approximation}~through~\ref{sec:lower}.

\section{Preliminaries}\label{sec:preliminaries}
The binary outcomes and the predictions are denoted by $x_1, x_2, \ldots, x_T$ and $p_1, p_2, \ldots, p_T$. We use the shorthands $x_{l:r}$ and $p_{l:r}$ for subsequences $x_l, x_{l+1},\ldots,x_r$ and $p_l, p_{l+1}, \ldots, p_r$. For $\alpha \in [0, 1]$ and $t \in \{0, 1, 2, \ldots, T\}$, $\Delta_{\alpha}(t) \coloneqq \sum_{t'=1}^{t}(x_{t'} - p_{t'})\cdot\1{p_{t'} = \alpha}$ is the total bias that the forecaster incurs on prediction value $\alpha$ up to time $t$. We drop the argument $t$ when it is clear from the context.

\paragraph{Lower calibration distance.}
Below is a formal definition of the lower calibration distance, which is equivalent to the ``lower distance from calibration'' defined by~\cite{BGHN23} up to a normalization factor of $T$.

\begin{definition}[Lower Calibration Distance]\label{def:lower-cal-dist}
The lower calibration distance of predictions $p \in [0, 1]^T$ with respect to outcomes $x \in \{0, 1\}^T$ is
\[
    \lowercaldist(x, p) \coloneqq \inf_{\D \in \lowerCal(x)}\sum_{t=1}^{T}\Ex{q_t \sim \D_t}{|p_t - q_t|},
\]
where $\lowerCal(x)$ is the family of $T$-tuples of distributions $\D = (\D_1, \D_2, \ldots, \D_T)$ such that: (1) The support of each $\D_t$ is finite and contained in $[0, 1]$; (2) $\D$ is perfectly calibrated with respect to $x$ in the sense that $\sum_{t=1}^{T}(x_t - \alpha)\cdot\D_t(\alpha) = 0$ holds for every $\alpha \in [0, 1]$.
\end{definition}
We will use the shorthand $\|p - \D\|_1 \coloneqq \sum_{t=1}^{T}\Ex{q_t \sim \D_t}{|p_t - q_t|}$ for $p \in [0, 1]^T$ and distributions $\D_1, \ldots, \D_T$ over $[0, 1]$. The definition above can then be simplified to $\inf_{\D \in \lowerCal(x)}\|p - \D\|_1$.

\begin{remark}\label{remark:optimal-transport}
    Definition~\ref{def:lower-cal-dist} becomes more natural when viewed through the lens of optimal transport. Imagine that, for each $t \in [T]$, there is one unit of bit $x_t$ located at point $p_t$. Then, each distribution $\D_t$ specifies a way of splitting and transporting the mass to (finitely many) different locations on $[0, 1]$. The distributions $\D_1, \ldots, \D_T$ are in the family $\lowerCal(x)$ if and only if after all the transportations are done, at every location $\alpha \in [0, 1]$, the fraction of ones is exactly $\alpha$. (Indeed, the constraint $\sum_{t=1}^{T}(x_t - \alpha)\cdot\D_t(\alpha) = 0$ is equivalent to $\frac{\sum_{t=1}^{T}x_t\cdot\D_t(\alpha)}{\sum_{t=1}^{T}\D_t(\alpha)} = \alpha$.) The lower calibration distance is exactly the minimum cost of the transportation subject to the calibration constraint on the resulting configuration, when the cost of moving one unit of mass from $p_t$ to $q_t$ is given by $|p_t - q_t|$.

    In comparison, the definition of the calibration distance introduces an additional constraint: each unit of mass cannot be transported to multiple locations, i.e., each $\D_t$ must be a degenerate distribution. This immediately gives the inequality $\lowercaldist(x, p) \le \caldist(x, p)$.
\end{remark}

\paragraph{Smooth calibration error.} Another related calibration measure is the \emph{smooth calibration error} proposed by~\cite{KF08}:
\[
    \smCE(x, p) \coloneqq 
    \sup_{f \in \F}\sum_{t=1}^{T}f(p_t)(x_t - p_t)
    =   \sup_{f \in \F}\sum_{\alpha \in [0, 1]}f(\alpha)\cdot\Delta_{\alpha}(T),
\]
where $\F$ is the family of $1$-Lipschitz functions from $[0, 1]$ to $[-1, 1]$.

It was shown by~\cite{BGHN23} that the smooth calibration error and the lower calibration distance are at most a constant factor away.
\begin{lemma}[Theorem~7.3 of \cite{BGHN23}]\label{lemma:smCE-vs-lowercaldist}
    For any $x \in \{0, 1\}^T$ and $p \in [0, 1]^T$,
    \[\frac{1}{2}\smCE(x, p) \le \lowercaldist(x, p) \le 2\smCE(x, p).\]
\end{lemma}

\paragraph{Lipschitz continuity.} Unlike the ECE, $\caldist(x, p)$, $\lowercaldist(x, p)$, and $\smCE(x, p)$ are all Lipschitz in the predictions $p$. Formally, we prove in Appendix~\ref{sec:preliminaries-proofs} that, for fixed $x \in \{0, 1\}^T$ and with respect to the $1$-norm, $\caldist(x, p)$ and $\lowercaldist(x, p)$ are $1$-Lipschitz, while $\smCE(x, p)$ is $2$-Lipschitz.

\section{Proof Overview}\label{sec:proof-overview}
We sketch the proofs of our results in this section. The proof of the approximation guarantees (Theorem~\ref{thm:lowercaldist-vs-caldist}) is the most involved and consists of several technical ingredients, for which we give an overview in Section~\ref{sec:overview-approx}. Given this approximation guarantee, the $O(\sqrt{T})$ upper bound (Theorem~\ref{thm:upper}) follows from a minimax argument and a reduction to online learning, outlined in Section~\ref{sec:overview-upper}.

Both the $\polylog(T)$ upper bound for random bits (Proposition~\ref{prop:random-bits-polylog}) and the $\Omega(T^{1/3})$ lower bound based on random bits and early stopping (Theorem~\ref{thm:lower}) are based on abstracting the setup as a ``controlled random walk'' game, which we define in Section~\ref{sec:overview-upper-random-bits}. We will discuss how to solve the game with a $\polylog(T)$ cost, and why that translates into an upper bound on the calibration distance. Finally, in Section~\ref{sec:overview-lower}, we explain why a connection in the other direction also holds, and how the $\Omega(T^{1/3})$ lower bound follows.

\subsection{Approximation Guarantees}\label{sec:overview-approx}
To show that $\lowercaldist(x, p)$ is a good approximation of $\caldist(x, p)$, we first pick $\D = (\D_1, \D_2, \ldots, \D_T) \in \lowerCal(x)$ as a ``witness'' of $\lowercaldist(x, p)$, i.e., $\lowercaldist(x, p) = \|p - \D\|_1$.\footnote{Technically, we can only find $\D$ that achieves $\|p - \D\|_1 \le \lowercaldist(x, p) + \eps$ for some $\eps > 0$. Since $\eps$ can be made arbitrarily small, the rest of the argument would not be affected.} Then, we round the distributions $\D_1, \ldots, \D_T$ to deterministic values $q_1, \ldots, q_T \in [0, 1]$ such that $q \in \Cal(x)$ and $\|p - q\|_1 \le \alpha\|p - \D\|_1 + \beta$. The desired approximation guarantee would then follow from
\[
    \caldist(x, p)
\le \|p - q\|_1
\le \alpha\|p - \D\|_1 + \beta
=   \alpha\lowercaldist(x, p) + \beta.
\]

This rounding is done in two steps. First, we transform $\D$ into another sequence $\D' = (\D'_1, \ldots, \D'_T)$ of $T$ distributions, such that: (1) $D' \in \lowerCal(x)$, i.e., $\D'$ is still perfectly calibrated; (2) There is a small finite set $S \subset [0, 1]$ that contains the support of every $\D'_t$; (3) $\|p - \D'\|_1$ can be upper bounded in terms of $\|p - \D\|_1$.

Concretely, Lemma~\ref{lemma:sparse-destination-general} gives such a transformation that guarantees
\[
    \|p - \D'\|_1 \le \|p - \D\|_1 + O(\sqrt{T})
\quad\text{and}\quad
    |S| = O(\sqrt{T}).
\]
When $p$ contains at most $m$ different entries, Lemma~\ref{lemma:sparse-destination} shows that we can alternatively achieve
\[
    \|p - \D'\|_1 \le O(1) \cdot \|p - \D\|_1
\quad\text{and}\quad
    |S| = O(m).
\]
The final ingredient is a method of rounding distributions $\D'_1, \ldots, \D'_T$ over set $S$ to deterministic values $(q_1, \ldots, q_T) \in \Cal(x)$. In Lemma~\ref{lemma:sparse-rounding}, we give a rounding scheme with the guarantee $\|p - q\|_1 \le \|p - \D'\|_1 + O(|S|)$. Clearly, Lemmas \ref{lemma:sparse-rounding}~through~\ref{lemma:sparse-destination} together prove Theorem~\ref{thm:lowercaldist-vs-caldist}.

\paragraph{General reduction to small support size.} For the general case, we prove Lemma~\ref{lemma:sparse-destination-general} using a simple binning strategy. Recall from Remark~\ref{remark:optimal-transport} that $\D = (\D_1, \ldots, \D_T)$ specifies a way of transporting $T$ units of mass (labeled with either $0$ or $1$) over the interval $[0, 1]$. For each $t$, a unit amount of bit $x_t$ is originally located at $p_t$, and gets transported according to distribution $\D_t$. The transportation incurs a total cost of $\|p - \D\|_1$, and the condition $\D \in \lowerCal(x)$ requires that, in the resulting configuration, the fraction of ones at each $\alpha \in [0, 1]$ is exactly $\alpha$.

A priori, the bits might be transported to many different destinations. We partition the interval $[0, 1]$ into $\sqrt{T}$ intervals with equal lengths. For the $i$-th interval $\left[\frac{i-1}{\sqrt{T}}, \frac{i}{\sqrt{T}}\right]$, we examine the mass being transported to all the locations within the interval. We \emph{consolidate} these transportations by redirecting them to a single destination, which is chosen such that the calibration constraint is still satisfied. Clearly, there will be at most $\sqrt{T}$ different destinations after the consolidation for all the intervals. Since $\D$ is perfectly calibrated, it is easy to show that the new destination falls into the same interval as the original destinations do, and is thus at a distance $\le 1/\sqrt{T}$. The total increase in the transportation cost will be bounded by $T \cdot (1/\sqrt{T}) = \sqrt{T}$ as desired.

\paragraph{Reduction to small support size under sparsity.} In the setup of Lemma~\ref{lemma:sparse-destination}, each $p_t$ is one of the $m$ values $s_1 < s_2 < \cdots < s_m$. In light of the proof strategy for the general case, it is tempting to try the following: Divide $[0, 1]$ into $m + 1$ intervals by splitting at each $s_i$. For each interval $[s_i, s_{i+1}]$, again, we consolidate all the transportations into the interval by redirecting them to a single location. Unfortunately, this does not work, since a typical interval $[s_i, s_{i+1}]$ has length $\Omega(1/m)$ and the redirection could incur an $\Omega(T/m)$ cost, which is too large.

In our proof, we still examine the mass being transported into the interval $[s_i, s_{i+1}]$ according to $\D$. Since each unit of mass originates at some $p_t \in \{s_1, s_2, \ldots, s_m\}$, the origin must be in $[0, s_i] \cup [s_{i+1}, 1]$. An important simplifying observation is that we may assume that all the bits originate from either $s_i$ or $s_{i+1}$. This is because the transportation of mass from some origin $p_t \in [0, s_i]\cup[s_{i+1}, 1]$ to a destination inside $[s_i, s_{i+1}]$ can be viewed as a two-phase process: first, transport the mass from $p_t$ to one of the endpoints ($s_i$ if $p_t \in [0, s_i]$ and $s_{i+1}$ if $p_t \in [s_{i+1}, 1]$); then, transport it from the endpoint to the actual destination. We will keep the first phase of each transportation unchanged, and focus on consolidating the second phases, in which the origins are either $s_i$ or $s_{i+1}$. We will ensure that, after the consolidation, there are $O(1)$ different destinations for each interval $[s_i, s_{i+1}]$, while the total cost of the second phases increases by at most an $O(1)$ factor. To find such a consolidation strategy, we exploit the connection between the lower calibration distance and smooth calibration error (Lemma~\ref{lemma:smCE-vs-lowercaldist}), and perform a quite involved case analysis.

\paragraph{Rounding of distributions supported over a small set.} Finally, we sketch the proof of our rounding lemma (Lemma~\ref{lemma:sparse-rounding}). The starting point is a sequence of $T$ distributions $\D_1, \ldots, \D_T$ over a common set $S$ of a small size. Let $s_1 < s_2 < \cdots < s_{|S|}$ be the elements of $S$. Suppose that for some $t_1 \ne t_2$, we have $x_{t_1} = x_{t_2}$, $p_{t_1} < p_{t_2}$. Meanwhile, $\D_{t_1}(s_i)$ and $\D_{t_2}(s_j)$ are both positive for some $i > j$. Intuitively, this means that $\D$ is inefficient---if we redirect an $\eps$ probability mass of $\D_{t_1}$ from $s_i$ to $s_j$, and the same amount in $\D_{t_2}$ from $s_j$ to $s_i$, we would end up with the same outcome without increasing the cost. In general, we should expect $\D$ to satisfy the following monotonicity property, or we can tweak it without increasing $\|p - \D\|_1$: for any $t_1, t_2$ such that $x_{t_1} = x_{t_2}$ and $p_{t_1} < p_{t_2}$, every element in the support of $\D_{t_1}$ is less than or equal to every element in the support of $\D_{t_2}$. In fact, this is a simple characterization of the optimal transport on a line.

Once we enforce this monotonicity, the rounding is easy---simply because there will not be much for us to round! Indeed, whenever $\D_t$ has a support of size at least $2$ (say, $\{s_{i}, s_{i+1}\}$), step $t$ must be, among all $t'$ such that $x_{t'} = x_t$ and $\D_{t'}(s_{i+1}) > 0$, the one with the smallest value of $p_{t'}$. This shows that $\D_t$ is degenerate, except for $O(|S|)$ different choices of $t$. For each non-degenerate distribution $\D_t$, we na\"ively pick $q_t = x_t$, so that calibration is satisfied. We also need to change the non-degenerate distributions to maintain calibration. It turns out that this rounding incurs an additional cost of $O(|S|)$, as desired.

\subsection{Calibration Distance Upper Bound}\label{sec:overview-upper}
Theorem~\ref{thm:lowercaldist-vs-caldist} and Lemma~\ref{lemma:smCE-vs-lowercaldist} together give
\[
    \caldist(x, p)
\le \lowercaldist(x, p) + O(\sqrt{T})
\le 2\smCE(x, p) + O(\sqrt{T}).
\]
Therefore, to prove Theorem~\ref{thm:upper}, it suffices to achieve an $O(\sqrt{T})$ smooth calibration error.

Suppose that the forecaster and the adversary are playing a zero-sum game, with the objective being the smooth calibration error. By the minimax theorem\footnote{Technically, we need to restrict the predictions to a finite set to apply the minimax theorem. This is handled by rounding the predictions to a $1/T$-net of $[0, 1]$ and applying Lipschitz continuity in the formal proof.}, we may assume that the adversary's (mixed) strategy is known. Then, at each step $t \in [T]$, we may compute the probability for the adversary to play $x_t = 1$ conditioning on $x_{1:(t-1)}$ and $p_{1:(t-1)}$. A natural strategy is then to choose $p_t$ as this conditional probability. To analyze the smooth calibration error incurred by this strategy, we frame this game as an instance of online learning on the class $\F$ of Lipschitz functions from $[0, 1]$ to $[-1, 1]$, and apply a regret bound in the online learning literature. The $O(\sqrt{T})$ bound follows from the fact that $\F$ has an $O(\sqrt{T})$ \emph{sequential Rademacher complexity}, which is an analogue of the usual Rademacher complexity for the sequential setup.

\subsection{Improved Forecasters for Random Bits}\label{sec:overview-upper-random-bits}
When the adversary commits to producing $T$ random bits, the minimization of the smooth calibration error is, informally, captured by the following control problem:

\begin{quote}
    \textbf{Controlled Random Walk:} The player starts at location $X_0 = 0$. At each step $t \in [T]$, the player first moves by $\eps_t \in \left[-\frac{1}{2}, \frac{1}{2}\right]$ (which may depend on $X_{t-1}$). Then, the nature perturbs the player's location by $\delta_t \in \left\{\pm \frac{1}{2}\right\}$ chosen uniformly at random. In other words, the player is located at $X_t = X_{t-1} + \eps_t + \delta_t$ after step $t$. The \emph{cost} of the player is defined as $|X_T| + \sum_{t=1}^{T}\eps_t^2$.  What is the lowest possible expected cost?
\end{quote}

To see how the game defined above is related to the calibration setup, recall that the smooth calibration error can be written as $\sup_{f \in \F}\sum_{\alpha \in [0, 1]}f(\alpha)\cdot\Delta_{\alpha}$, where $\Delta_{\alpha} = \sum_{t=1}^{T}(x_t - p_t)\cdot\1{p_t = \alpha}$ is the total bias associated with prediction value $\alpha$, and $\F$ is the family of $1$-Lipschitz functions from $[0, 1]$ to $[-1, 1]$. For any $f \in \F$, we have
\begin{align*}
    \sum_{\alpha \in [0, 1]}f(\alpha)\cdot\Delta_{\alpha}
&=  f\left(1/2\right)\cdot\sum_{\alpha \in [0, 1]}\Delta_{\alpha} + \sum_{\alpha \in [0, 1]}\left[f(\alpha) - f\left(1/2\right)\right]\cdot\Delta_{\alpha}\\
&\le \left|\sum_{t=1}^{T}(x_t - p_t)\right| + \sum_{\alpha \in [0, 1]}\left|\alpha - 1/2\right|\cdot|\Delta_{\alpha}|.
\end{align*}
If we write $p_t = 1/2 - \eps_t$ and $x_t = 1/2 + \delta_t$, the first term above reduces to $\left|\sum_{t=1}^{T}(\eps_t + \delta_t)\right| = |X_T|$, the first term in the cost of the player. In the second term, we note that for each $\alpha \in [0, 1]$, the expectation of $\Delta_{\alpha} = \sum_{t=1}^{T}(x_t - p_t)\cdot\1{p_t = \alpha}$ is exactly $(1/2 - \alpha)$ times the expected number of times $\alpha$ is predicted. If we ``assume'' that $\Ex{}{|\Delta_{\alpha}|}$ is equal to the absolute value of $\Ex{}{\Delta_{\alpha}}$\footnote{This is clearly false in general. In the actual proof, we use a concentration argument on $\Delta_{\alpha}$ to show that this approximation is essentially true.}, the second term can be equivalently written as
\[
    \sum_{\alpha \in [0, 1]}\left|\alpha - 1/2\right|^2\cdot\sum_{t=1}^{T}\1{p_t = \alpha}
=   \sum_{t=1}^{T}\left(p_t - 1/2\right)^2
=   \sum_{t=1}^{T}\eps_t^2.
\]
Therefore, an upper bound on the cost in the controlled random walk game gives a uniform upper bound on $\sum_{\alpha \in [0, 1]}f(\alpha)\cdot\Delta_{\alpha}$ over all $f \in \F$ and thus, by definition, upper bounds $\smCE(x, p)$.

\paragraph{A strategy with sub-$\sqrt{T}$ cost.} The trivial strategy of playing $\eps_t = 0$ at every step gives a cost of $\Ex{}{|X_T|} = \Ex{}{\left|\sum_{t=1}^{T}\delta_t\right|} = \Theta(\sqrt{T})$. Can we do better? Consider the following simple strategy: Fix a parameter $\eps \in (0, 1/2]$, and play $\eps_t = -\eps \cdot \sgn(X_{t-1})$ at step $t$. In other words, we move towards the origin by a distance of $\eps$ at each step. This strategy clearly gives $\sum_{t=1}^{T}\eps_t^2 \le T\eps^2$.

For the $|X_T|$ term, the following heuristic argument suggests $\Ex{}{|X_T|} = O(1/\eps)$. Assume that the 
noise $\delta_t$ follows the standard Gaussian instead of the uniform distribution on $\{\pm 1/2\}$. Then, the random process $X_t = X_{t-1} - \eps\cdot\sgn(X_t) + \delta_t$ is a discretization of the following dynamics:
\[
    \frac{\rmd X(t)}{\rmd t} = -\nabla U(X(t)) + \rmd B(t),
\]
where the potential is $U(x) = \eps|x|$, and $B(t)$ is the standard Brownian motion. As $t \to +\infty$, we expect the distribution of $X_t$ to converge to a distribution with density at $x$ proportional to $e^{-\beta U(x)}$ for some constant $\beta$, i.e., the Laplace distribution. This implies $\Ex{}{|X_t|} = O(1/\eps)$. 

If we set $\eps = T^{-1/3}$, both terms in the cost would be bounded by $O(T^{1/3})$, which gives a polynomial improvement over the trivial cost of $\sqrt{T}$. Formalizing this heuristic argument gives a forecasting strategy with $\Ex{}{\smCE(x, p)} = O(T^{1/3})$ against random bits. Since the number of different predictions (i.e., the number of different values of $\eps_t$) is a constant, applying the second part of Theorem~\ref{thm:lowercaldist-vs-caldist} shows that $\caldist(x, p)$ is also at most $O(T^{1/3})$ in expectation.

\paragraph{A strategy with $\polylog(T)$ cost.} The improvement from $T^{1/3}$ to $\polylog(T)$ is done by varying the parameter $\eps$ throughout the game. Interestingly, unlike the usual ``doubling trick'' in online learning, our time horizon is divided into epochs with geometrically \emph{decreasing} lengths.

Suppose that in the first $T/2$ steps, the player simply drifts with the noise. Typically, we expect to have $|X_{T/2}| = O(\sqrt{T})$. For concreteness, assume that $X_{T/2} = \sqrt{T}$. Starting from time $T/2 + 1$, we play $\eps_t = -\alpha / \sqrt{T}$ for some $\alpha = \polylog(T)$. Then, we expect that $X_t$ will drop below $0$ before the game ends with high probability. This is because, if the game (hypothetically) runs for $T/2$ more steps, $X_T$ would roughly follow a Gaussian with mean $\sqrt{T} - \frac{T}{2}\cdot\frac{\alpha}{\sqrt{T}} = -\Omega(\alpha\sqrt{T})$ and variance $O(T)$. For sufficiently large $\alpha$, this will be negative with high probability. Therefore, the player simply waits for $X_t$ to become negative, at which point $X_t$ should be very close to $0$. Now, there are at most $T/2$ steps remaining, and we repeat the same strategy for the rest of the game.

This strategy clearly controls $X_t$ such that $|X_T| = O(1)$ with high probability. To upper bound the $\sum_{t=1}^{T}\eps_t^2$ term, note that in the first epoch, we predict a value with absolute value $\alpha/\sqrt{T}$ at most $T / 2$ times. This contributes at most $\frac{T}{2} \cdot (\alpha / \sqrt{T})^2 = \polylog(T)$ to the sum. Since we repeat this at most $O(\log T)$ times, we end up with a $\polylog(T)$ cost. Finally, since the procedure only involves $O(\log T)$ different values of $\eps_t$, applying the second bound in Theorem~\ref{thm:lowercaldist-vs-caldist} bounds the calibration distance by $\polylog(T)$ as well.

\subsection{Calibration Distance Lower Bound}\label{sec:overview-lower}
Our lower bound proof is based on the observation that a lower bound for the controlled random walk game also gives a lower bound on the smooth calibration error. Again, we write $p_t = 1/2 - \eps_t$ and $x_t = 1/2 + \delta_t$. Then, the location $X_T$ of the player after $T$ steps is exactly given by $X_T = \sum_{t=1}^{T}(\eps_t + \delta_t) = \sum_{t=1}^{T}(x_t - p_t)$, the difference between the total outcomes and total predictions. Recall that $\smCE(x, p)$ is the supremum of $\sum_{t=1}^{T}f(p_t)\cdot(x_t - p_t)$ among all $1$-Lipschitz functions $f: [0, 1] \to [-1, 1]$. In particular, by picking $f$ to be the constant function $1$ or $-1$, we obtain $\smCE(x, p) \ge |X_T|$.

To see how the $\sum_{t=1}^{T}\eps_t^2$ term comes into play, consider the function $f(x) = 1 / 2 - x$, which is in the family $\F$. This gives $\sum_{t=1}^{T}f(p_t)\cdot(x_t - p_t) = \sum_{t=1}^{T}\eps_t\cdot(\eps_t + \delta_t)$. After taking an expectation, the $\eps_t\cdot\delta_t$ term vanishes. Therefore, at least in expectation, the smooth calibration error is lower bounded by the $\sum_{t=1}^{T}\eps_t^2$ term as well.

However, as outlined in Section~\ref{sec:overview-upper-random-bits}, the player can achieve a $\polylog(T)$ cost in the controlled random walk game. To obtain the $\Omega(T^{1/3})$ lower bound, we make another simple observation: as long as we can lower bound the expectation of $\max_{t \in [T]}|X_t| + \sum_{t=1}^{T}\eps_t^2$ by $\Omega(T^{1/3})$, we can obtain the same lower bound in the prediction setting via an \emph{early stopping} trick, which was used by~\cite{QV21} in their lower bound on the ECE.

Indeed, if an algorithm gives $\Ex{}{\sum_{t=1}^{T}\eps_t^2} = \Omega(T^{1/3})$, the connection that we made earlier lower bounds $\Ex{}{\smCE(x, p)}$ by $\Omega(T^{1/3})$ as desired. Otherwise, the value $|X_t|$ must be large at some point $t'$. Equivalently, the bias $\left|\sum_{t=1}^{t'}(x_t - p_t)\right|$ is large. Then, if the adversary deviates from outputting random bits, and keep outputting the same bit, this large bias will remain in the end.

To lower bound this strengthened cost of $\max_{t \in [T]}|X_t| + \sum_{t=1}^{T}\eps_t^2$, we divide the horizon $T$ into $T^{1/3}$ \emph{epochs} of length $T^{2/3}$ each. In a typical block, the sum of $\delta_t$ has an absolute value of $\Omega(\sqrt{T^{2/3}}) = \Omega(T^{1/3})$. Then, if the total control of the player (i.e., the sum of $\eps_t$) is much smaller than $T^{1/3}$ in absolute value, we will catch a large $|X_t|$ during this epoch. In order not to be caught, the player is forced to ensure that the sum of $\eps_t$ is $\pm \Omega(T^{1/3})$ in a typical epoch. This, in turn, lower bounds the sum of $\eps_t^2$ within that epoch by $\Omega(1)$. Summing over the $T^{1/3}$ epochs gives the desired lower bound of $\sum_{t=1}^{T}\eps_t^2 = \Omega(T^{1/3})$.

\section{Discussion and Open Problems}\label{sec:discussion}

\paragraph{Is the calibration distance a good metric?} In this work, we propose to use the calibration distance as a calibration measure in sequential prediction setups. The definition of the calibration distance is natural, and in the same spirit as the work of~\cite{BGHN23} for the offline setup. Compared to the ECE, the calibration distance is better-behaved in being Lipschitz continuous in the predictions. Compared to alternative calibration measures that are continuous (such as weak and smooth calibration), the calibration distance is, from the forecaster's perspective, especially easy and intuitive to certify---to show that the calibration distance is small, the forecaster only needs to output a set of alternative predictions that are calibrated and close to the actual predictions. From this perspective, our proof of Theorem~\ref{thm:lowercaldist-vs-caldist} is \emph{algorithmic} in the sense that it implies an efficient algorithm for the forecaster to find such a certificate.

On the other hand, the calibration distance is still far from being the ``perfect'' calibration measure in the sequential setup. As we highlight in Proposition~\ref{prop:random-bits-polylog}, even on a sequence of random bits, minimizing the calibration distance might incentivize the forecaster to deviate from the ``right'' predictions. Unfortunately, this is unavoidable to some extent---as discussed in Section~\ref{sec:related-work}, such incentive-related issues may arise even when only perfect calibration is concerned.

\paragraph{Stronger approximation guarantees.} An obvious open problem is to strengthen Theorem~\ref{thm:approx} and design better approximation algorithms for the calibration distance. A natural avenue is through refining the structural results that relate the calibration distance to the lower calibration distance. More concretely, is the $O(\sqrt{T})$ upper bound on the gap between $\caldist(x, p)$ and $\lowercaldist(x, p)$ tight? Can we avoid the extra $O(1)$ multiplicative factor for the sparse case?

\paragraph{Explicit and efficient algorithms.} Our proof of Theorem~\ref{thm:upper} is based on the minimax theorem and thus non-constructive. Deriving an actual algorithm requires solving a zero-sum game with an action space that is doubly-exponential in $T$. Is there an explicit and efficient algorithm for matching the $O(\sqrt{T})$ guarantee? A concrete approach is based on the prior work of~\cite{KF08,FH18,FH21}, which gave \emph{deterministic} forecasters that asymptotically satisfy weak or smooth calibration. Roughly speaking, their forecasting algorithms are deterministic because they are based on fixed point theorems rather than the minimax theorem. While these work focused on asymptotic calibration rather than the exact convergence bounds, it follows easily from \cite[Lemma 4.3]{KF08} that their algorithm gives $\Ex{}{\smCE(x, p)} = O(T^{2/3})$ in our notations. Is there a more refined analysis of their approach that gives an $O(\sqrt{T})$ bound? Can we efficiently implement their algorithms, which, as stated, are based on finding fixed points?

\paragraph{Stronger lower bounds.} Another obvious open problem is to strengthen the $\Omega(T^{1/3})$ lower bound, which is essentially based a sequence of random bits (each taking value $1$ with probability $1/2$). A natural attempt is to divide the time horizon into multiple epochs, and use different probabilities for different epochs. The issue is, of course, that the forecaster may predict strategically to decrease the error that it has accumulated in previous epochs.  For the ECE, this difficulty was partially resolved by~\cite{QV21} via a ``sidestepping'' technique, which uses an adaptive, divide-and-conquer strategy for choosing the probabilities for different epochs. Can we use the same approach to bootstrap the $T^{1/3}$ lower bound to a rate closer to $T^{1/2}$?

\section{Proof of the Approximation Guarantees}\label{sec:approximation}
In this section, we prove Theorem~\ref{thm:lowercaldist-vs-caldist}, which states that the lower calibration distance is a good additive approximation of the calibration distance. Then, Theorem~\ref{thm:approx} follows easily from the algorithm of~\cite{BGHN23} that computes the lower calibration distance up to an $\eps$ additive error in $\poly(n, 1/\eps)$ time.

\subsection{Impossibility of Multiplicative Approximation}\label{sec:no-multiplicative-approx}
Before we proceed to the proof, we give a concrete example showing that $\lowercaldist(x, p)$ does not give a good multiplicative approximation of $\caldist(x, p)$, even if the multiplicative factor is allowed to depend on $T$ (but not on $x$ or $p$).

\begin{proposition}\label{prop:multiplicative-approx}
There is no function $f:\bbN \to (0, +\infty)$ such that the following holds for all $T$, $x \in \{0, 1\}^T$ and $p \in [0, 1]^T$:
\begin{equation}\label{eq:multiplicative-approx}
    \lowercaldist(x, p) \ge f(T)\cdot\caldist(x, p).
\end{equation}
\end{proposition}

The proof is based on an example inspired by the proof of~\cite[Lemma 4.5]{BGHN23}. 

\begin{proof}
Let $T = 4$, $x = (0, 1, 0, 1)$, and $p = (1/2 - \eps, 1/2 - \eps, 1/2 + \eps, 1/2 + \eps)$ for some small $\eps \in (0, 1/30)$. We will show that $\caldist(x, p) = 4\eps$ while $\lowercaldist(x, p) = O(\eps^2)$. This implies that the ratio
\[
    \frac{\lowercaldist(x, p)}{\caldist(x, p)}
\le \frac{O(\eps^2)}{4\eps}
=   O(\eps)
\]
can be made arbitrarily small by taking $\eps \to 0^{+}$, so Inequality~\eqref{eq:multiplicative-approx} cannot hold for any fixed $f$.

\paragraph{The calibration distance.} We first note that $\caldist(x, p) \le 4\eps$, as $q = (1/2, 1/2, 1/2, 1/2)$ is perfectly calibrated with respect to $x$, and $\|p - q\|_1 = 4\eps$. Furthermore, for any $q \in \Cal(x)$, all the entries of $q$ must lie in
\[
    \{a/b: a \in \{0,1, \ldots, b\}, b \in [4]\} = \{0, 1/4, 1/3, 1/2, 2/3, 3/4, 1\}.
\]
If any entry $q_i$ is different from $1/2$, the difference $|p_i - q_i|$ must be at least $(1/2 - \eps) - 1/3 = 1/6 - \eps \ge 4\eps$ (the last step follows from $\eps < 1/30$), which implies $\|p - q\|_1 \ge 4\eps$. This shows $\caldist(x, p) = 4\eps$.

\paragraph{The lower calibration distance.} Roughly speaking, the $O(\eps^2)$ lower calibration distance is achieved by transporting $O(\eps)$ units of the ``extra ones (resp., zeros)'' at $1/2 - \eps$ (resp., $1/2 + \eps$) to $1/2$. Formally, let $\D_1$ and $\D_4$ be the degenerate distributions supported on $\{1/2 - \eps\}$ and $\{1/2 + \eps\}$, respectively. Let $\beta = \frac{1/2 - \eps}{1/2 + \eps}$, and define $\D_2$ and $\D_3$ as
\[
    \D_2(1/2 - \eps) = \D_3(1/2 + \eps) = \beta, \quad\quad
    \D_2(1/2) = \D_3(1/2) = 1 - \beta.
\]
We can then verify that $(\D_1, \D_2, \D_3, \D_4) \in \lowerCal(x)$, because for $\alpha = 1/2$, we have
\[
    \sum_{t=1}^{T}(x_t - \alpha) \cdot \D_t(\alpha)
=   (x_2 - 1/2)\cdot (1 - \beta) + (x_3 - 1/2)\cdot(1 - \beta)
=   \frac{1-\beta}{2} - \frac{1-\beta}{2}
=   0,
\]
and for $\alpha = 1/2 - \eps$, we have
\[
    \sum_{t=1}^{T}(x_t - \alpha) \cdot \D_t(\alpha)
=   (x_1 - 1/2 + \eps)\cdot 1 + (x_2 - 1/2 + \eps)\cdot \beta
=   0.
\]
The $\alpha = 1/2 + \eps$ case holds by symmetry. By definition of the lower calibration distance,
\[
    \lowercaldist(x, p)
\le \sum_{t=1}^{T}\Ex{q_t \sim \D_t}{|p_t - q_t|}
=   0 + (1 - \beta)\cdot\eps + (1 - \beta)\cdot \eps + 0
=   \frac{4\eps^2}{1/2 + \eps}
=   O(\eps^2).
\]
\end{proof}

\subsection{Rounding of Distributions with a Small Support}
We start with the following lemma, which is crucial for proving both bounds in Theorem~\ref{thm:lowercaldist-vs-caldist}.

\begin{lemma}\label{lemma:sparse-rounding}
    Suppose that $x \in \{0, 1\}^T$, $p \in [0, 1]^T$, and $\D = (\D_1, \D_2, \ldots, \D_T) \in \lowerCal(x)$, where $\D_1, \ldots, \D_T$ are distributions supported over a finite set $S \subset [0, 1]$. Then,
    \[
        \caldist(x, p) \le \|p - \D\|_1 + 4|S|.
    \]
\end{lemma}

The lemma states that if we have distributions $\D_1$ through $\D_T$ supported on a common set of a small size, and they serve as a witness for $\lowercaldist(x, p)$ being small, we can ``round'' them to a sequence of (deterministic) predictions and show that $\caldist(x, p)$ is small as well. The rounding procedure only leads to an additive increase in the distance that is linear in the support size.

\begin{proof}
    We will first transform $\D_1, \ldots, \D_T$ to another $T$ distributions $(\D'_1, \ldots, \D'_T) \in \lowerCal(x)$ that satisfy a monotonicity property. Furthermore, the new distributions are still over set $S$, and the cost $\|p - \D'\|_1$ does not exceed the original cost $\|p - \D\|_1$. With this monotonicity, we apply a simple rounding scheme to produce a good witness $q \in \Cal(x)$ which shows that $\caldist(x, p)$ is small.

    \paragraph{Enforce monotonicity.} For $b \in \{0, 1\}$, let $\T^{(b)} \coloneqq \{t \in [T]: x_t = b\}$ denote the set of time steps where the outcome is $b$. We claim that there exist distributions $\D'_1, \ldots, \D'_T$ over $S$ such that:
    \begin{itemize}
        \item $\D' = (\D'_1, \ldots, \D'_T) \in \lowerCal(x)$.
        \item $\|p - \D'\|_1 \le \|p - \D\|_1$.
        \item There exist total orders on $\T^{(0)}$ and $\T^{(1)}$ (both denoted by ``$\prec$'') such that: For every $b \in \{0, 1\}$ and $t_1, t_2 \in \T^{(b)}$, $t_1 \prec t_2$ implies: (1) $p_{t_1} \le p_{t_2}$; (2) the maximum element in the support of $\D'_{t_1}$ is less than or equal to the minimum element in the support of $\D'_{t_2}$.
    \end{itemize}
    In words, the third condition requires that among all the time steps in $\T^{(b)}$, steps with a small value of $p_t$ corresponds to a distribution $\D'_t$ supported over smaller values.

    \paragraph{Construction of $\D'$.} The existence of $\D'$ should be obvious when viewing the problem as an optimal transport in one dimension: For each $b \in \{0, 1\}$, we originally have one unit of mass on $p_t$ for each $t \in \T^{(b)}$, while the goal is to obtain the configuration $\sum_{t \in \T^{(b)}}\D_t$. In order to minimize the total cost, we should match the two measures in increasing order.

    Nevertheless, we provide an elementary proof of this claim by explicitly constructing the distributions $\D'_1$ through $\D'_T$. For each $b \in \{0, 1\}$, let $m \coloneqq \left|\T^{(b)}\right|$ and $t_1, t_2, \ldots, t_m$ be a permutation of the elements of $\T^{(b)}$ such that $p_{t_1} \le p_{t_2} \le \cdots \le p_{t_m}$. Then, we set $\D'_{t_1}, \D'_{t_2}, \ldots, \D'_{t_m}$ as the unique distributions over set $S$ that satisfy:
    \begin{itemize}
        \item $\sum_{i=1}^{m}\D'_{t_i} = \sum_{i=1}^{m}\D_{t_i}$, i.e., for every $\alpha \in S$, $\sum_{i=1}^{m}\D'_{t_i}(\alpha) = \sum_{i=1}^{m}\D_{t_i}(\alpha)$.
        \item For every $i \in [m-1]$, the maximum element in the support of $\D'_{t_i}$ is less than or equal to the minimum element in the support of $\D'_{t_{i+1}}$.
    \end{itemize}
    More concretely, these $m$ distributions can be found by starting with the total measure $\sum_{i=1}^{m}\D_{t_i}$, and then greedily forming a probability measure by taking one unit of mass from the smallest elements in the support of the remaining measure, until $m$ probability measures are formed.

    By construction, the distributions $\D'_1, \ldots, \D'_T$ are over set $S$, and satisfy the monotonicity constraint (with respect to the total order defined as $t_1 \prec t_2 \prec \cdots \prec t_m$). It remains to show that $\D' \in \lowerCal(x)$ and that the cost of $\D'$ is not higher than that of $\D$.
    
    \paragraph{$\D'$ is perfectly calibrated.} We note that for every $\alpha \in [0, 1]$,
    \[
        \sum_{t=1}^{T}(x_t - \alpha)\D'_t(\alpha)
    =   \sum_{b \in \{0, 1\}}(b - \alpha)\sum_{t \in \T^{(b)}}\D'_t(\alpha)
    =   \sum_{b \in \{0, 1\}}(b - \alpha)\sum_{t \in \T^{(b)}}\D_t(\alpha)
    =   \sum_{t=1}^{T}(x_t - \alpha)\D_t(\alpha)
    =   0.
    \]
    The second step follows from our construction of $\D'$, while the last step holds since $\D \in \lowerCal(x)$. This proves $\D' \in \lowerCal(x)$.

    \paragraph{$\D'$ does not have a higher cost.} Fix $b \in \{0, 1\}$. Let $a_0 < a_1 < \cdots < a_m$ be the distinct elements in $\{p_1, \ldots, p_T\} \cup S$. For each $i \in [m]$, define
    \[
        F_i \coloneqq \sum_{t \in \T^{(b)}}\1{p_t < a_i} 
        \quad \text{and} \quad
        G_i \coloneqq \sum_{t \in \T^{(b)}}\pr{q_t \sim \D_t}{q_t < a_i}.
    \]
    
    We claim that
    \begin{equation}\label{eq:EMD}
        \sum_{t \in \T^{(b)}}\Ex{q_t \sim \D'_t}{|p_t - q_t|}
    =   \sum_{i=1}^{m}(a_i - a_{i-1})\cdot|F_i - G_i|
    \le \sum_{t \in \T^{(b)}}\Ex{q_t \sim \D_t}{|p_t - q_t|}.
    \end{equation}
    Summing over $b \in \{0, 1\}$ proves $\|p - \D'\|_1 \le \|p - \D\|_1$.

    We first prove the first step in Equation~\eqref{eq:EMD}. For each $i \in [m]$, let $\delta_i$ denote the amount of mass transported across the interval $[a_{i-1}, a_i]$ according to $\D'$. Formally, we define
    \[
        \delta_i
    \coloneqq
        \sum_{t \in \T^{(b)}}\pr{q_t \sim \D'_t}{[a_{i-1}, a_i] \subseteq [\min\{p_t, q_t\}, \max\{p_t, q_t\}]}.
    \]
    For any $p, q \in \{a_0, a_1, \ldots, a_m\}$, we have the identity
    \[
        |p - q|
    =   \sum_{i=1}^{m}(a_i - a_{i-1})\cdot\1{[a_{i-1}, a_i] \subseteq [\min\{p, q\}, \max\{p, q\}]}.
    \]
    Thus, we can re-write the cost $\sum_{t \in \T^{(b)}}\Ex{q_t \sim \D'_t}{|p_t - q_t|}$ as $\sum_{i=1}^{m}(a_i - a_{i-1})\cdot\delta_i$, and it remains to prove that $\delta_i = |F_i - G_i|$ for every $i \in [m]$.
    
    Fix $i \in [m]$. Let $t_1 \prec t_2 \prec \cdots$ be the elements of $\T^{(b)}$ sorted according to total order $\prec$. Recall that $F_i$ is the number of values among $\{p_t: t \in T^{(b)}\}$ that are strictly smaller than $a_i$, so we have $p_{t_j} \le a_{i-1}$ for every $j \le F_i$ and $p_{t_j} \ge a_i$ for every $j > F_i$. 
    
    Suppose that $F_i \ge G_i$. By construction of $\D'_t$, for every $j \le \lfloor G_i \rfloor$, the support of $\D'_{t_j}$ is contained in $[0, a_{i-1}]$ while $p_{t_j} \in [0, a_{i-1}]$, so they do not contribute to $\delta_i$. For $j \in \{\lceil G_i \rceil + 1, \ldots, F_i\}$, the support of $\D'_{t_j}$ is completely contained in $[a_i, 1]$ while $p_{t_j} \in [0, a_{i-1}]$, so they contribute $F_i - \lceil G_i \rceil$ to $\delta_i$. Finally, when $G_i$ is not integral, for $j = \lceil G_i \rceil$, we have $p_{t_j} \in [0, a_{i-1}]$ and $\D'_{t_j}$ assigns a probability mass of $\lceil G_i \rceil - G_i$ to $[a_i, 1]$. This contributes $\lceil G_i \rceil - G_i$ to $\delta_i$. Therefore, we conclude that, in this case,
    \[
        \delta_i = (F_i - \lceil G_i \rceil) + (\lceil G_i \rceil - G_i) = |F_i - G_i|.
    \]

    The $F_i < G_i$ case is similar. For $j \le F_i$, we have $p_{t_j} \in [0, a_{i-1}]$, while the support of $\D'_{t_j}$ is also contained in $[0, a_{i-1}]$, so these values of $j$ do not contribute to $\delta_i$. When $F_i + 1 \le j \le \lfloor G_i\rfloor$, the support of $\D'_{t_j}$ is still contained in $[0, a_{i-1}]$, whereas $p_{t_j} \ge a_i$. This contributes $\lfloor G_i\rfloor - F_i$ to $\delta_i$. Finally, when $G_i$ is not integral, for $j = \lceil G_i \rceil$, $\D'_{t_j}$ assigns a probability mass of $G_i - \lfloor G_i \rfloor$ to $[0, a_{i-1}]$, and this contributes $G_i - \lfloor G_i \rfloor$ to $\delta_i$. Again, we have $\delta_i = G_i - F_i = |F_i - G_i|$.

    Next, we prove the second step in Equation~\eqref{eq:EMD}, i.e., $\sum_{t \in \T^{(b)}}\Ex{q_t \sim \D_t}{|p_t - q_t|}$ is lower bounded by $\sum_{i=1}^{m}(a_i - a_{i-1})\cdot|F_i - G_i|$. Similarly, we define
    \[
        \delta_i
    \coloneqq
        \sum_{t \in \T^{(b)}}\pr{q_t \sim \D_t}{[a_{i-1}, a_i] \subseteq [\min\{p_t, q_t\}, \max\{p_t, q_t\}]}
    \]
    as the total mass transported across $[a_{i-1}, a_i]$ according to $\D$, and it suffices to show that $\delta_i \ge |F_i - G_i|$ for every $i \in [m]$.

    Fix $i \in [m]$ and suppose that $F_i \ge G_i$. Note that for any $p, q \in \{a_0, a_1, \ldots, a_m\}$, we have the inequality
    \[
        \1{p \le a_{i-1} \wedge q \ge a_i}
    \ge \1{p < a_i} - \1{q < a_i}.
    \]
    Then, by definition of $\delta_i$, we have
    \[
        \delta_i
    \ge \sum_{t \in \T^{(b)}}\pr{q_t \sim \D_t}{p_t \le a_{i-1} \wedge q_t \ge a_i}
    \ge \sum_{t \in \T^{(b)}}\1{p_t < a_i} - \sum_{t \in \T^{(b)}}\pr{q_t \sim \D_t}{q_t < a_i}
    =   |F_i - G_i|.
    \]
    Similarly, when $F_i < G_i$, using the inequality
    \[
        \1{p \ge a_i \wedge q \le a_{i-1}}
    \ge \1{p \ge a_i} - \1{q \ge a_i},
    \]
    we get
    \begin{align*}
        \delta_i
    \ge \sum_{t \in \T^{(b)}}\pr{q_t \sim \D_t}{p_t \ge a_i \wedge q_t \le a_{i-1}}
    &\ge\sum_{t \in \T^{(b)}}\1{p_t \ge a_i} - \sum_{t \in \T^{(b)}}\pr{q_t \sim \D_t}{q_t \ge a_i}\\
    &=  \left(\left|\T^{(b)}\right| - F_i\right) - \left(\left|\T^{(b)}\right| - G_i\right)
    =   |F_i - G_i|.
    \end{align*}
    This concludes the proof of the inequality $\|p - \D'\|_1 \le \|p - \D\|_1$, and shows that $\D'$ indeed has all the desired properties.

    \paragraph{The rounding scheme.} Now that the distributions $\D'_1, \ldots, \D'_T$ have all the nice properties, it remains to find $q \in \Cal(x)$ such that $\|p - q\|_1 \le \|p - \D'\|_1 + 4|S|$, since the lemma would then follow from
    \[
        \caldist(x, p)
    \le \|p - q\|_1
    \le \|p - \D'\|_1 + 4|S|
    \le \|p - \D\|_1 + 4|S|.
    \]

    Let $\Tpure$ be the set of indices $t \in [T]$ such that $\D'_t$ is degenerate (i.e., with a size-$1$ support). We call each $t \in \Tpure$ a ``pure step'', and each $t \in \Tmixed \coloneqq [T] \setminus \Tpure$ a ``mixed step''. For each pure step $t \in \Tpure$, let $\beta_t \in [0, 1]$ be the (only) element in the support of $\D'_t$.
    
    We choose $q \in [0, 1]^T$ as follows:
    \begin{itemize}
        \item For each pure step $t \in \Tpure$, set $q_t$ to
        \[
            g(\beta_t)
        \coloneqq
            \frac{\sum_{t' \in \Tpure}x_{t'}\cdot \1{\beta_{t'} = \beta_t}}{\sum_{t' \in \Tpure}\1{\beta_{t'} = \beta_t}}.
        \]
        \item For each mixed step $t \in \Tmixed$, set $q_t = x_t$.
    \end{itemize}
    In the remainder of the proof, we verify that $q \in \Cal(x)$ and then upper bound $\|p - q\|_1$.
    
    \paragraph{$q$ is perfectly calibrated.} For every $\alpha \in [0, 1]$, we can write
    \[
        \sum_{t=1}^{T}(x_t - q_t)\cdot\1{q_t = \alpha}
    =   \sum_{t \in \Tpure}(x_t - q_t)\cdot\1{q_t = \alpha} + \sum_{t \in \Tmixed}(x_t - q_t)\cdot\1{q_t = \alpha}.
    \]
    The second summation is $0$, since $q_t = x_t$ holds for every $t \in \Tmixed$. By our choice of $q$, the first summation can be written as
    \begin{align*}
        &~\sum_{t \in \Tpure}(x_t - q_t)\cdot\1{g(\beta_t) = \alpha}\\
    =   &~\sum_{\alpha' \in S}\1{g(\alpha') = \alpha}\sum_{t \in \Tpure}(x_t - \alpha)\cdot\1{\beta_t = \alpha'}\\
    =   &~\sum_{\alpha' \in S}\1{g(\alpha') = \alpha}\cdot\left[\sum_{t \in \Tpure}x_t\cdot\1{\beta_t = \alpha'} - g(\alpha')\cdot\sum_{t \in \Tpure}\1{\beta_t = \alpha'}\right]\\
    =   &~0,
    \end{align*}
    where the last step follows from the definition of $g(\cdot)$. This verifies $q \in \Cal(x)$.
    
    \paragraph{Upper bound $\|p - q\|_1$.} Note that
    \begin{align*}
        \|p - q\|_1
    &=  \sum_{t \in \Tpure}|p_t - q_t| + \sum_{t \in \Tmixed}|p_t - q_t|\\
    &\le\sum_{t \in \Tpure}|p_t - \beta_t| + \sum_{t \in \Tpure}|q_t - \beta_t| + \left|\Tmixed\right|\\
    &\le\sum_{t=1}^{T}\Ex{q_t \sim \D'_t}{|p_t - q_t|} + \sum_{t \in \Tpure}|q_t - \beta_t| + \left|\Tmixed\right|.
    \end{align*}
    In the following, we will show that both $\sum_{t \in \Tpure}|q_t - \beta_t|$ and $\left|\Tmixed\right|$ are upper bounded by $2|S|$, which would conclude the proof.

    \paragraph{Bound the number of mixed steps.} We start by showing $|\Tmixed| \le 2|S|$. Fix $t \in \Tmixed$. Let $b = x_t$ and $s \in S$ be the largest element in the support of $\D'_t$. Since $\D'_t$ is not degenerate, the support of $\D'_t$ contains another element $s' < s$.  Recall that $\T^{(b)}$ is associated with a total order $\prec$ that is consistent with both $p_t$'s and the supports of $\D'_t$. Then, with respect to order $\prec$, $t$ must be the smallest index in $\T^{(b)}$ such that $\D'_t(s) \ne 0$. Indeed, if there exists $t' \prec t$ such that $\D'_{t'}(s) > 0$, the fact that the support of $\D'_t$ contains a smaller element $s' < s$ contradicts the monotonicity. Therefore, we showed that every $t \in \Tmixed$ corresponds to a unique pair $(b, s) \in \{0, 1\}\times S$. This implies $|\Tmixed| \le 2|S|$.

    \paragraph{Bound the additional cost on the pure steps.} Fix $s \in S$. For each $b \in \{0, 1\}$, let
    \[
        n_b \coloneqq \sum_{t \in \Tpure}\1{\beta_t = s \wedge x_t = b}
    \quad \text{and} \quad
        \eps_b \coloneqq \sum_{t \in \Tmixed}\D'_t(s)\cdot\1{x_t = b}.
    \]
    We claim that
    \[
        s = \frac{n_1 + \eps_1}{n_0 + n_1 + \eps_0 + \eps_1}
    \]
    and
    \[
        g(s) = \frac{n_1}{n_0 + n_1}.
    \]
    The latter follows immediately from the definition of $g(\cdot)$, $n_0$ and $n_1$. The former holds since $\D' \in \lowerCal(x)$ implies
    \begin{align*}
        0
    &=  \sum_{t=1}^{T}(x_t - s)\cdot\D'_t(s)\\
    &=  \sum_{t \in \Tpure}(x_t - s)\cdot\D'_t(s) + \sum_{t \in \Tmixed}(x_t - s)\cdot\D'_t(s)\\
    &=  [n_1 - s\cdot (n_0 + n_1)] + [\eps_1 - s\cdot(\eps_0 + \eps_1)],
    \end{align*}
    which, after rearrangement, gives the expression of $s$.

    We also argue that $\eps_0, \eps_1 \in [0, 2]$. Fix $b \in \{0, 1\}$, and let $t_1, t_2 \in \T^{(b)}$ be the smallest and the largest index $t \in \T^{(b)}$ (with respect to total order $\prec$) such that $\D'_t(s) \ne 0$. By monotonicity, for any $t \in \T^{(b)}$ such that $t_1 \prec t \prec t_2$, the support of $\D'_t$ can only contain $s$, which impies $t \in \Tpure$. Therefore, $\eps_b$ is exactly given by $\sum_{t \in \{t_1, t_2\}}\D'_t(s)$, which clearly lies in $[0, 2]$.

    Then, we have
    \[
        (n_0 + n_1) \cdot \left|s - g(s)\right|
    =   (n_0 + n_1) \cdot \left|\frac{n_1 + \eps_1}{n_0 + n_1 + \eps_0 + \eps_1} - \frac{n_1}{n_0 + n_1}\right|.
    \]
    For fixed $n_0$ and $n_1$, the last expression is maximized when either $(\eps_0, \eps_1) = (2, 0)$ or $(\eps_0, \eps_1) = (0, 2)$. A simple calculation shows that the expression is upper bounded by $2$ in both cases.

    Finally, we note that
    \[
        \sum_{t \in \Tpure}|q_t - \beta_t|
    =   \sum_{s \in S}\sum_{t \in \Tpure}|q_t - \beta_t|\cdot\1{\beta_t = s}
    =   \sum_{s \in S}|g(s) - s|\cdot\sum_{t \in \Tpure}\1{\beta_t = s}.
    \]
    Therefore, the contribution of each $s \in S$ to $\sum_{t \in \Tpure}|q_t - \beta_t|$ is exactly $(n_0 + n_1)\cdot|s - g(s)| \le 2$.
    Thus, we have $\sum_{t \in \Tpure}|q_t - \beta_t| \le 2|S|$.
\end{proof}

\subsection{Proof of the Additive Gap}
With Lemma~\ref{lemma:sparse-rounding} in hand, we prove the first part of Theorem~\ref{thm:lowercaldist-vs-caldist}, which upper bounds the gap between $\caldist(x, p)$ and $\lowercaldist(x, p)$ by $O(\sqrt{T})$.

The proof starts by finding $T$ distributions $\hat\D_1, \ldots, \hat\D_T$ that (approximately) achieve the lower calibration distance $\lowercaldist(x, p)$. We refine these distributions to $\D_1, \ldots, \D_T$, so that: (1) the support of every $\D_t$ is contained in the same set of size $O(\sqrt{T})$; (2) $\D$ still approximately achieves $\lowercaldist(x, p)$ up to an $O(\sqrt{T})$ slack. This allows us to invoke our rounding lemma (Lemma~\ref{lemma:sparse-rounding}) to show $\caldist(x, p) \le \lowercaldist(x, p) + O(\sqrt{T})$.

Before proceeding with the proof below, it would be helpful to recall the connection between the lower calibration distance and optimal transport (Remark~\ref{remark:optimal-transport}). In particular, during the proof we will sometimes refer to the distributions $(\D_1, \ldots, \D_T) \in \lowerCal(x)$ and the corresponding transportation of the bits interchangeably.

\begin{lemma}\label{lemma:sparse-destination-general}
    For any $x \in \{0, 1\}^T$, $p \in [0, 1]^T$, there exists a set $S \subseteq [0, 1]$ of size at most $O(\sqrt{T})$ along with distributions $\D_1, \ldots, \D_T$ over $S$, such that $(\D_1, \D_2, \ldots, \D_T) \in \lowerCal(x)$ and
    \[
        \|p - \D\|_1 \le \lowercaldist(x, p) + O(\sqrt{T}).
    \]
\end{lemma}

\begin{proof}
    By definition of the lower calibration distance, there exists $\hat\D = (\hat\D_1, \hat\D_2, \ldots, \hat\D_T) \in \lowerCal(x)$ such that $\|p - \hat\D\|_1 \le \lowercaldist(x, p) + 1$.\footnote{We need the ``$+1$'' term in case that the infimum in the definition of $\lowercaldist(x, p)$ cannot be achieved.}

    Pick an integer $m = \Theta(\sqrt{T})$ and define intervals $\Int_i \coloneqq [(i-1)/m, i/m)$ for $i \in [m - 1]$ and $\Int_m \coloneqq [(m-1)/m, 1]$. Roughly speaking, for each $i \in [m]$, we will examine the bits that are transported into the interval $\Int_i$ according to $\hat\D$. We then consolidate these bits into a single destination. The perfect calibration of $\hat\D$ implies that the new destination still falls into $\Int_i$, so our change increases the cost by at most $T / m = O(\sqrt{T})$. Furthermore, the new transportation only involves at most $m$ different destinations (one for each interval $\Int_i$).

    \paragraph{The new destinations.} For each $i \in [m]$, we define
    \[
        \mu_i \coloneqq \frac{\sum_{t=1}^{T}x_t\cdot\hat\D_t(\Int_i)}{\sum_{t=1}^{T}\hat\D_t(\Int_i)}
    =   \frac{\sum_{\alpha \in \Int_i}\sum_{t=1}^{T}x_t\cdot\hat\D_t(\alpha)}{\sum_{\alpha \in \Int_i}\sum_{t=1}^{T}\hat\D_t(\alpha)}.
    \]
    Note that by definition of $\lowerCal(x)$, each $\hat\D_t$ has a finite support, so both $\sum_{t=1}^{T}x_t\cdot\hat\D_t(\alpha)$ and $\sum_{t=1}^{T}\hat\D_t(\alpha)$ take non-zero values only on finitely many choices of $\alpha$. Therefore, the summations over $\alpha$ in the last expression above are actually finite and thus well-defined.
    
    We argue that for every $i \in [m]$, $\mu_i$ falls into $[(i-1)/m, i/m]$. Indeed, we can re-write $\mu_i$ as
    \begin{equation}\label{eq:mu-convex-comb}
        \mu_i
    =   \sum_{\alpha \in \Int_i}\frac{\sum_{t=1}^{T}x_t\cdot\hat\D_t(\alpha)}{\sum_{\beta \in \Int_i}\sum_{t=1}^{T}\hat\D_t(\beta)}
    =   \sum_{\alpha \in \Int_i}\frac{\sum_{t=1}^{T}\hat\D_t(\alpha)}{\sum_{\beta \in \Int_i}\sum_{t=1}^{T}\hat\D_t(\beta)}\cdot \frac{\sum_{t=1}^{T}x_t\cdot\hat\D_t(\alpha)}{\sum_{t=1}^{T}\hat\D_t(\alpha)}.
    \end{equation}
    For each $\alpha \in \Int_i$, it follows from $\hat\D \in \lowerCal(x)$ that $\sum_{t=1}^{T}(x_t - \alpha)\cdot\hat\D_t(\alpha) = 0$
    and, equivalently, $\frac{\sum_{t=1}^{T}x_t \cdot\hat\D_t(\alpha)}{\sum_{t=1}^{T}\hat\D_t(\alpha)} = \alpha \in \Int_i$. Therefore, Equation~\eqref{eq:mu-convex-comb} states that $\mu_i$ is a convex combination of values that lie in interval $\Int_i$, which ensures that $(i-1)/m \le \mu_i \le i / m$.

    \paragraph{The updated transportation.} We define another $T$ distributions $\D_1, \D_2, \ldots, \D_T$ as follows. Let $\phi:[0, 1]\to[0,1]$ be the function that maps every value in each $\Int_i$ to $\mu_i$. Then, each $\D_t$ is defined as the distribution of $\phi(q_t)$ when $q_t \sim \hat\D_t$. We will argue that $\D = (\D_1, \ldots, \D_T)$ is in $\lowerCal(x)$ and that $\|p - \D\|_1$ is comparable to the cost $\|p - \hat\D\|_1$.

    To show that $\D \in \lowerCal(x)$, we note that for any $\alpha \in [0, 1]$, we have
    \begin{align*}
        \sum_{t=1}^{T}(x_t - \alpha)\cdot\D_t(\alpha)
    &=  \sum_{t=1}^{T}(x_t - \alpha)\cdot\pr{q_t \sim \hat\D_t}{\phi(q_t) = \alpha} \tag{definition of $\D_t$}\\
    &=  \sum_{t=1}^{T}(x_t - \alpha)\cdot\sum_{i=1}^{m}\1{\mu_i = \alpha}\cdot\hat\D_t(\Int_i) \tag{definition of $\phi$}\\
    &=  \sum_{i=1}^{m}\1{\mu_i = \alpha}\cdot\left[\sum_{t=1}^{T}x_t\cdot\hat\D_t(\Int_i) - \mu_i \cdot \sum_{t=1}^{T}\hat\D_t(\Int_i)\right]
    =   0. \tag{definition of $\mu_i$}
    \end{align*}
    This proves $\D \in \lowerCal(x)$.

    To show the latter property, we note that since $\phi(x)$ and $x$ always fall into the same interval $\Int_i$, we have $|\phi(x) - x| \le 1 / m$ for every $x \in [0, 1]$. Then,
    \begin{align*}
        \sum_{t=1}^{T}\Ex{q_t \sim \D_t}{|p_t - q_t|}
    &=  \sum_{t=1}^{T}\Ex{q_t \sim \hat\D_t}{|p_t - \phi(q_t)|} \tag{definition of $\D_t$}\\
    &\le\sum_{t=1}^{T}\Ex{q_t \sim \hat\D_t}{|p_t - q_t|} + \sum_{t=1}^{T}\Ex{q_t \sim \hat\D_t}{|q_t - \phi(q_t)|} \tag{triangle inequality}\\
    &\le[\lowercaldist(x, p) + 1] + T\cdot \frac{1}{m} \tag{choice of $\hat\D$ and $|x - \phi(x)| \le 1/m$}\\
    &\le\lowercaldist(x, p) + O(\sqrt{T}). \tag{$m = \Theta(\sqrt{T})$}
    \end{align*}
    Finally, note that each $\D_t$ is over the same set of size $\le m = O(\sqrt{T})$, namely, $\{\mu_1, \mu_2, \ldots, \mu_m\}$. This concludes the proof.
\end{proof}

The first bound in Theorem~\ref{thm:lowercaldist-vs-caldist} then follows easily.

\begin{proof}[Proof of the first part of Theorem~\ref{thm:lowercaldist-vs-caldist}]
    By Lemma~\ref{lemma:sparse-destination-general}, there exists $\D = (\D_1, \D_2, \ldots, \D_T) \in \lowerCal(x)$ such that $\|p - \D\|_1 \le \lowercaldist(x, p) + O(\sqrt{T})$, and each $\D_t$ is over the same set of size $O(\sqrt{T})$. By Lemma~\ref{lemma:sparse-rounding}, we have
    \[
        \caldist(x, p)
    \le \|p - \D\|_1 + O(\sqrt{T})
    \le \lowercaldist(x, p) + O(\sqrt{T}).
    \]
\end{proof}

\subsection{Approximation Guarantee in the Sparse Case}\label{sec:approx-sparse}
Now we deal with the second part of Theorem~\ref{thm:lowercaldist-vs-caldist}, where we have a prediction sequence $p \in [0, 1]^T$ with only $m$ different entries. In order to invoke our Lemma~\ref{lemma:sparse-rounding}, however, we need to show that the lower calibration distance $\lowercaldist(x, p)$ can be approximately achieved by distributions $\D_1, \ldots, \D_T$ over a small set $S \subset [0, 1]$ (more concretely, of size $O(m)$). This step, stated as the lemma below, turns out to be much more complicated.

\begin{lemma}\label{lemma:sparse-destination}
    For any $x \in \{0, 1\}^T$, $p \in [0, 1]^T$, $m = |\{p_1, p_2, \ldots, p_T\}|$ and $\eps > 0$, there exists a set $S \subset [0, 1]$ of size at most $2m + 3$ along with distributions $\D_1, \ldots, \D_T$ over $S$, such that
    \[
        (\D_1, \D_2, \ldots, \D_T) \in \lowerCal(x)
    \]
    and
    \[
        \|p - \D\|_1 \le 20\cdot\lowercaldist(x, p) + \eps.
    \]
\end{lemma}

We first show how Lemmas \ref{lemma:sparse-rounding}~and~\ref{lemma:sparse-destination} together imply the second part of Theorem~\ref{thm:lowercaldist-vs-caldist}.

\begin{proof}[Proof of the second part of Theorem~\ref{thm:lowercaldist-vs-caldist}]
    Applying Lemma~\ref{lemma:sparse-destination} with $\eps = 1$ gives a set $S \subset [0, 1]$ of size $\le 2m + 3$, together with distributions $\D_1, \ldots, \D_T$ over $S$ such that $\D \in \lowerCal(x)$ and 
    \[
        \|p - \D\|_1 \le 20 \cdot \lowercaldist(x, p) + 1.
    \]
    Then, by Lemma~\ref{lemma:sparse-rounding},
    \[
        \caldist(x, p)
    \le \|p - \D\|_1 + 4 \cdot (2m + 3)
    \le 20 \cdot \lowercaldist(x, p) + (8m + 13).
    \]
\end{proof}

Before we prove Lemma~\ref{lemma:sparse-destination}, again, we recommend the reader to review Remark~\ref{remark:optimal-transport}. In the following proof, we will frequently mention the transportation of the bits in lieu of the explicit expressions for the probability distributions.

As in the proof of the first part of Theorem~\ref{thm:lowercaldist-vs-caldist}, we start by picking $\hat\D \in \lowerCal(x)$ that approximately achieves $\lowercaldist(x, p)$, and then consolidate the transportation specified by $\hat\D$ such that there will be at most $O(m)$ destinations. 

Let $s_1 < s_2 < \cdots < s_m$ be the $m$ values that appear in the entries of $p$. A natural first attempt would be to examine the bits that are transported into each interval $[s_i, s_{i+1}]$ and merge them to a single destination. Unfortunately, as we show in Appendix~\ref{sec:approximation-proofs}, this na\"ive consolidation could blow up the cost. Instead, our proof of Lemma~\ref{lemma:sparse-destination} involves a much more complicated case analysis based on the amounts of zeros and ones being transported into $[s_i, s_{i+1}]$ on both directions.

\begin{proof}[Proof of Lemma~\ref{lemma:sparse-destination}]
    By definition of $\lowercaldist(x, p)$, there exist $T$ distributions $\hat\D_1, \hat\D_2, \ldots, \hat\D_T$ such that:
    \begin{itemize}
        \item Each $\hat\D_t$ is supported over a finite subset of $[0, 1]$.
        \item $\hat\D_1, \ldots, \hat\D_T$ are perfectly calibrated, i.e., $\sum_{t=1}^{T}(x_t - \alpha)\cdot\hat\D_t(\alpha) = 0$ holds for every $\alpha \in [0, 1]$.
        \item The cost is close to $\lowercaldist(x, p)$: $\|p - \hat\D\|_1 \le \lowercaldist(x, p) + \eps / 20$.\footnote{Again, we need the ``$\eps/20$'' term because the infimum in the definition of $\lowercaldist(x, p)$ might not be achieved by any $\D \in \lowerCal(x)$.}
    \end{itemize}

   \paragraph{Proof overview.} Let $0 = s_1 < s_2 < \cdots < s_{m'} = 1$ be the distinct values among $\{p_1, p_2, \ldots, p_T\} \cup \{0, 1\}$. Note that $m' \le |\{p_1, p_2, \ldots, p_T\}| + 2 = m + 2$. We will transform $\hat\D_1, \ldots, \hat\D_T$ into another $T$ distributions, denoted by $\D_1, \ldots, \D_T$, over a set $S \subset [0, 1]$, such that:
   \begin{itemize}
       \item $|S| \le 2m'-1 \le 2m + 3$;
       \item $(\D_1, \ldots, \D_T) \in \lowerCal(x)$;
       \item $\|p - \D\|_1 \le 20\cdot\|p - \hat\D\|_1$.
    \end{itemize}
    Note that doing so would prove the lemma, since the last condition implies that
    \[
        \|p - \D\|_1
    \le 20\|p - \hat\D\|_1
    \le 20 \cdot (\lowercaldist(x, p) + \eps / 20)
    =   20 \cdot \lowercaldist(x, p) + \eps.
    \]

    To ensure the first property, we examine the probability masses that $\hat \D_1$ through $\hat \D_T$ assign to the interval $(s_i, s_{i+1})$ for each $i \in [m' - 1]$. These can be interpreted as a way of transporting certain fractions of the bits $x_1, \ldots, x_T$ to the interval, so that the resulting configuration is calibrated. A priori, the bits might be transported to many different destinations within the interval $(s_i, s_{i+1})$. We will reroute the transportation, so that the bits will only arrive at five different destinations: $0$, $1$, $s_i$, $s_{i+1}$, and another unique value assigned for this interval. In the end, the supports of $\D_1$ through $\D_T$ will be among $s_1, s_2, \ldots, s_{m'}$ along with $m'-1$ other values. Therefore, the corresponding set $S$ will have size at most $2m'-1$.

    \paragraph{Decomposition of costs.} Our first step is to decompose the cost $\|p - \hat\D\|_1$ into a few parts. Let $\Int_i \coloneqq [s_i, s_{i+1})$ for every $i \in [m' - 2]$ and $\Int_{m'-1} \coloneqq [s_{m'-1}, s_{m'}] = [s_{m'-1}, 1]$. The total cost associated with interval $\Int_i$ is defined as
    \[
        \cost_i \coloneqq \sum_{t=1}^{T}\Ex{q_t \sim \hat \D_t}{|p_t - q_t|\cdot\1{q_t \in \Int_i}}.
    \]
    Furthermore, for each interval $\Int_i$, we decompose the cost according to whether the transportation is from the left or from the right (we view the interval $[0, 1]$ as a line segment in which the small values lie on the left):
    \begin{align*}
        &~\costl_i \coloneqq \sum_{t=1}^{T}\Ex{q_t \sim \hat \D_t}{|p_t - q_t|\cdot\1{q_t \in \Int_i\wedge p_t \le s_i}},\\
        &~\costr_i \coloneqq \sum_{t=1}^{T}\Ex{q_t \sim \hat \D_t}{|p_t - q_t|\cdot\1{q_t \in \Int_i\wedge p_t \ge s_{i+1}}}.
    \end{align*}
    Finally, note that whenever the condition $q_t \in \Int_i \wedge p_t \le s_i$ holds in the definition of $\costl_i$, we have $p_t \le s_i \le q_t$, which gives $|p_t - q_t| = |p_t - s_i| + |q_t - s_i|$. Therefore, we can write $\costl_i = \costlo_i + \costli_i$, where
    \begin{align*}
        \costlo_i \coloneqq \sum_{t=1}^{T}\Ex{q_t \sim \hat \D_t}{|p_t - s_i|\cdot\1{q_t \in \Int_i\wedge p_t \le s_i}},\\
        \costli_i \coloneqq \sum_{t=1}^{T}\Ex{q_t \sim \hat \D_t}{|q_t - s_i|\cdot\1{q_t \in \Int_i\wedge p_t \le s_i}}.
    \end{align*}
    Here the superscripts ``$\sfo$'' and ``$\sfi$'' specify whether the cost is for the transportation outside or inside the interval $\Int_i$.
    Similarly, we decompose $\costr_i$ into the following two terms:
    \begin{align*}
        \costro_i \coloneqq \sum_{t=1}^{T}\Ex{q_t \sim \hat \D_t}{|p_t - s_{i+1}|\cdot\1{q_t \in \Int_i\wedge p_t \ge s_{i+1}}},\\
        \costri_i \coloneqq \sum_{t=1}^{T}\Ex{q_t \sim \hat \D_t}{|q_t - s_{i+1}|\cdot\1{q_t \in \Int_i\wedge p_t \ge s_{i+1}}}.
    \end{align*}
    Our use of the word ``decompose'' can be justified by the following identity:
    \begin{align*}
        \sum_{t=1}^{T}\Ex{q_t\sim\hat\D_t}{|p_t - q_t|}
    &=  \sum_{i=1}^{m'-1}\cost_i\\
    &=  \sum_{i=1}^{m'-1}(\costl_i + \costr_i)\\
    &=  \sum_{i=1}^{m'-1}(\costlo_i +\costli_i + \costro_i + \costri_i).
    \end{align*}
    The first step holds since $\Int_1, \ldots, \Int_{m'-1}$ form a partition of $[0, 1]$, which implies $1 = \sum_{i=1}^{m'-1}\1{x \in \Int_i}$ for any $x \in [0, 1]$. The second step follows from the observation that $p_t$ never falls into $(s_i, s_{i+1})$, so we have $\1{q_t \in \Int_i} = \1{q_t \in \Int_i \wedge p_t \le s_i} + \1{q_t \in \Int_i \wedge p_t \ge s_{i+1}}$.

    \paragraph{Lower bound the cost of the second phase.} For $b \in \{0, 1\}$, let $\unitl_{i,b}$ denote the amount of bit $b$ that is transported into interval $\Int_i$ from $[0, s_i]$. Formally,
    \[
        \unitl_{i,b}
    \coloneqq \sum_{t=1}^{T}\hat\D_t(\Int_i)\cdot\1{x_t = b \wedge p_t \le s_i}.
    \]
    Similarly, $\unitr_{i,b}$ is defined as the amount of bit $b$ moved from $[s_{i+1}, 1]$ to $\Int_i$:
    \[
        \unitr_{i,b}
    \coloneqq \sum_{t=1}^{T}\hat\D_t(\Int_i)\cdot\1{x_t = b \wedge p_t \ge s_{i+1}}.
    \]
    Intuitively, $\hat\D_1$ through $\hat\D_T$ specify the following transportation of bits:
    \begin{itemize}
        \item For each $i \in [m'-1]$, we spend a total cost of $\costlo_i + \costro_i$ to transport zeros and ones from $[0, s_i] \cup [s_{i+1}, 1]$ to either $s_i$ and $s_{i+1}$ (``the first phase'').
        \item At this point, there are $\unitl_{i,0}$ (resp., $\unitl_{i,1}$) units of zeros (resp., ones) at $s_i$, and $\unitr_{i,b}$ units of bit $b$ at $s_{i+1}$.
        \item Then, we further distribute these bits to values within $\Int_i$ so that the outcomes are calibrated (``the second phase''), at a total cost of $\costli_i + \costri_i$.
    \end{itemize}

    The distributions $\D_1, \ldots, \D_T$ that we will define is based on a new transportation that keeps the total cost of the first phase (outside of $\Int_i$). We will change the second phase, so that there will be at most one destination outside $\{s_1, s_2, \ldots, s_{m'}\}$. Furthermore, we make sure that this change only blows up the cost of the second phase by a constant factor.

    For this purpose, we start by lower bounding $\costli_i + \costri_i$. 
    Let $\Deltal_i \coloneqq \unitl_{i,1} - (\unitl_{i,0} + \unitl_{i,1}) \cdot s_i$ and $\Deltar_i \coloneqq \unitr_{i,1} - (\unitr_{i,0} + \unitr_{i,1}) \cdot s_{i+1}$ denote the biases incurred at point $s_i$ and $s_{i+1}$ between the first and the second phases. We will prove the following inequality: For any $1$-Lipschitz function $f:[0,1]\to[-1, 1]$,
    \begin{equation}\label{eq:cost-vs-smCE}
        2\left(\costli_i + \costri_i\right) \ge f(s_i)\cdot\Deltal_i + f(s_{i+1})\cdot\Deltar_i.
    \end{equation}
    The following proof is the same as the one in~\cite{BGHN23} for lower bounding the lower distance from calibration by the smooth calibration error. We include the proof for completeness.

    Fix a $1$-Lipschitz function $f:[0,1]\to[-1, 1]$. Consider the function $g_b(v) \coloneqq f(v)\cdot (b - v)$ defined over $[0, 1]$ for $b \in \{0, 1\}$. Since $|g_b'(v)| = |f'(v)(b - v) - f(v)| \le 2$ for any $v \in [0, 1]$, $g_b$ is $2$-Lipschitz. Then, we have
    \begin{align*}
        2\costli_i
    &=  \sum_{t=1}^{T}\Ex{q_t \sim \hat\D_t}{2|q_t - s_i|\cdot\1{q_t \in \Int_i \wedge p_t \le s_i}}\\
    &\ge\sum_{t=1}^{T}\Ex{q_t \sim \hat\D_t}{[f(s_i)\cdot(x_t - s_i) - f(q_t)\cdot(x_t - q_t)]\cdot\1{q_t \in \Int_i \wedge p_t \le s_i}} \tag{$v\mapsto f(v)\cdot(x_t - v)$ is $2$-Lipschitz}\\
    &=  \sum_{t=1}^{T}\Ex{q_t \sim \hat\D_t}{f(s_i)\cdot(x_t - s_i)\cdot\1{q_t \in \Int_i \wedge p_t \le s_i}}\\
    &- \sum_{t=1}^{T}\Ex{q_t \sim \hat\D_t}{f(q_t)\cdot(x_t - q_t)\cdot\1{q_t \in \Int_i \wedge p_t \le s_i}}.
    \end{align*}
    The first summation in the last expression above can be further simplified into:
    \begin{align*}
        &~f(s_i)\cdot\sum_{t=1}^{T}\Ex{q_t \sim \hat\D_t}{x_t\cdot\1{q_t \in \Int_i \wedge p_t \le s_i}} - f(s_i)\cdot s_i\cdot \sum_{t=1}^{T}\Ex{q_t \sim \hat\D_t}{\1{q_t \in \Int_i \wedge p_t \le s_i}}\\
    =   &~f(s_i)\cdot\sum_{t=1}^{T}\hat\D_t(\Int_i)\cdot\1{x_t = 1 \wedge p_t \le s_i} - f(s_i)\cdot s_i\cdot \sum_{t=1}^{T}\hat\D_t(\Int_i)\cdot\1{p_t \le s_i}\\
    =   &~f(s_i)\cdot[\unitl_{i,1} - s_i\cdot(\unitl_{i,0} + \unitl_{i,1})]
    =   f(s_i) \cdot \Deltal_i,
    \end{align*}
    and thus,
    \begin{equation}\label{eq:costli-lower-bound}
        2\costli_i
    \ge f(s_i)\cdot\Deltal_i - \sum_{t=1}^{T}\Ex{q_t \sim \hat\D_t}{f(q_t)\cdot(x_t - q_t)\cdot\1{q_t \in \Int_i \wedge p_t \le s_i}}.
    \end{equation}
    An analogous argument gives
    \begin{equation}\label{eq:costri-lower-bound}
        2\costri_i
    \ge f(s_{i+1})\cdot\Deltar_i - \sum_{t=1}^{T}\Ex{q_t \sim \hat\D_t}{f(q_t)\cdot(x_t - q_t)\cdot\1{q_t \in \Int_i \wedge p_t \ge s_{i+1}}}.
    \end{equation}
    Finally, Inequality~\eqref{eq:cost-vs-smCE} follows from Inequalities \eqref{eq:costli-lower-bound}~and~\eqref{eq:costri-lower-bound}, together with the observation that the two summations on the right-hand sides of \eqref{eq:costli-lower-bound}~and~\eqref{eq:costri-lower-bound} sum up to
    \begin{align*}
        \sum_{t=1}^{T}\Ex{q_t \sim \hat\D_t}{f(q_t)\cdot(x_t - q_t)\cdot\1{q_t \in \Int_i}}
    &=  \sum_{\alpha \in \Int_i}\sum_{t=1}^{T}\Ex{q_t \sim \hat\D_t}{f(q_t)\cdot(x_t - q_t)\cdot\1{q_t = \alpha}}\\
    &=  \sum_{\alpha \in \Int_i}f(\alpha)\cdot\sum_{t=1}^{T}(x_t - \alpha)\cdot\hat\D_t(\alpha)
    =   0,
    \end{align*}
    where the last step follows from $\hat D \in \lowerCal(x)$.

    Given Inequality~\eqref{eq:cost-vs-smCE}, we apply Lemma~\ref{lemma:two-point-smCE} from Appendix~\ref{sec:approximation-proofs} to lower bound $\costli_i + \costri_i$ by a closed-form expression of $\Deltal_i$, $\Deltar_i$, $s_i$, and $s_{i+1}$.
    \begin{equation}\label{eq:cost-vs-biases}
        2\left(\costli_i + \costri_i\right)
    \ge \begin{cases}
        |\Deltal_i| + |\Deltar_i|, & \Deltal_i\cdot\Deltar_i \ge 0,\\
        |\Deltal_i + \Deltar_i| + (s_{i+1} - s_i)\cdot\min\{|\Deltal_i|, |\Deltar_i|\}, & \Deltal_i\cdot\Deltar_i < 0.
    \end{cases}
    \end{equation}

    \paragraph{Handle the same-sign situation.} It remains to change the second phase of the transportation inside interval $\Int_i$, so that there will be at most one destination (in addition to $s_1, s_2, \ldots, s_{m'}$), while the cost is bounded by $20\cdot(\costli_i + \costri_i)$.

    We start by noting that we may assume $\min\left\{\unitl_{i,0}, \unitl_{i,1}\right\} = \min\left\{\unitr_{i,0}, \unitr_{i,1}\right\} = 0$ without loss of generality. This is because, for example, when both $\unitl_{i,0}$ and $\unitl_{i,1}$ are positive, we may take $\mu \coloneqq \min\{\unitl_{i,0} / (1 - s_i), \unitl_{i,1} / s_i\}$, and let $\mu\cdot s_i$ units of ones and $\mu\cdot(1 - s_i)$ units of zeros be ``settled'' at $s_i$. After this, either $\unitl_{i,0}$ or $\unitl_{i,1}$ becomes zero, and the quantity $\Deltal_i$ is unchanged. The same argument applies to $\unitr_{i,0}$ and $\unitr_{i,1}$ as well.

    We first deal with the case that $\Deltal_i, \Deltar_i \ge 0$. In this case, we have $\unitl_{i,0} = \unitr_{i,0} = 0$, i.e., there are no extra zeros at either $s_i$ or $s_{i+1}$, though there might be extra ones. We will transport these ones to $1$, at a cost of
    \[
        \unitl_{i,1}\cdot(1 - s_i) + \unitr_{i,1}\cdot(1 - s_{i+1})
    =   \Deltal_i + \Deltar_i
    =   |\Deltal_i| + |\Deltar_i|
    \le 2\left(\costli_i + \costri_i\right).
    \]
    The last step above follows from Equation~\eqref{eq:cost-vs-biases}. Similarly, if $\Deltal_i, \Deltar_i \le 0$, we have $\unitl_{i,1} = \unitr_{i,1} = 0$. We will transport all the extra zeros to $0$, and the total cost will be
    \[
        \unitl_{i, 0}\cdot s_i + \unitr_{i, 0}\cdot s_{i+1}
    =   -\Deltal_i - \Deltar_i
    =   |\Deltal_i| + |\Deltar_i|
    \le 2\left(\costli_i + \costri_i\right).
    \]

    In both cases, we settle all the bits that were originally associated with interval $\Int_i$ at a total cost of at most $2\cost_i$. Furthermore, all the destinations lie in the set $\{0, s_i, s_{i+1}, 1\}$.

    \paragraph{Handling opposite signs, the first part.} The case that $\Deltal_i \cdot \Deltar_i < 0$ is more involved. We first deal with the case that $\Deltal_i > 0$ and $\Deltar_i < 0$. Recall that we assumed $\min\{\unitl_{i,0}, \unitl_{i,1}\} = \min\{\unitr_{i,0}, \unitr_{i,1}\} = 0$ without loss of generality. This means that $\unitl_{i,1}, \unitr_{i,0} > 0$, while $\unitl_{i,0} = \unitr_{i,1} = 0$.

    We shorthand $x \coloneqq \unitl_{i,1}$ and $y \coloneqq \unitr_{i,0}$. Our strategy is to move all the bits---the $x$ units of ones at $s_i$ and the $y$ units of zeros at $s_{i+1}$---to value $p \coloneqq \frac{x}{x + y}$. The total cost would be
    \[
        x \cdot\left|p - s_i\right| + y\cdot\left|p - s_{i+1}\right|
    =   (x + y)\cdot \left[p\cdot|p - s_i| + (1 - p)\cdot|p - s_{i+1}|\right].
    \]

    Note that $\Deltal_i = x\cdot(1 - s_i)$ and $\Deltar_i = -y\cdot s_{i+1}$. The right-hand side of Inequality~\eqref{eq:cost-vs-biases} can then be re-written as
    \begin{align*}
        &~\left|x\cdot(1 - s_i) - y\cdot s_{i+1}\right| + (s_{i+1} - s_i)\cdot\min\{x\cdot(1 - s_i), y\cdot s_{i+1}\}\\
    =   &~(x + y)\cdot\left[\left|p\cdot(1-s_i) - (1-p)\cdot s_{i+1}\right| + (s_{i+1} - s_i)\cdot\min\{p\cdot(1-s_i), (1-p)\cdot s_{i+1}\}\right].
    \end{align*}

    Then, applying Lemma~\ref{lemma:technical-ineq-1} from Appendix~\ref{sec:approximation-proofs} with $\alpha = s_i$ and $\beta = s_{i+1}$ shows that the cost of the new transportation is at most
    \[
        2 \cdot \left[|\Deltal_i + \Deltar_i| + (s_{i+1} - s_i)\cdot\min\{|\Deltal_i|, |\Deltar_i|\}\right]
    \le 4\cdot\left(\costli_i + \costri_i\right).
    \]

    \paragraph{Handling opposite signs, the second part.} It remains to handle the case that $\Deltal_i < 0$ and $\Deltar_i > 0$. In this case, we have $\unitl_{i,0}, \unitr_{i,1} > 0$, while $\unitl_{i,1} = \unitr_{i,0} = 0$.

    Again, shorthand $x \coloneqq \unitl_{i,0}$ and $y \coloneqq \unitr_{i,1}$. The key difference is that we will consider the following two strategies, and use the one with a lower cost:
    \begin{itemize}
        \item \textbf{Strategy 1:} Again, move all the bits---the $x$ units of zeros at $s_i$ and the $y$ units of ones at $s_{i+1}$---to value $p \coloneqq \frac{y}{x + y}$. The total cost would be
        \[
            x \cdot\left|\frac{y}{x+y} - s_i\right| + y\cdot\left|\frac{y}{x+y} - s_{i+1}\right|
        =   (x + y)\cdot \left[(1 - p)\cdot|p - s_i| + p\cdot|p - s_{i+1}|\right].
        \]
        \item \textbf{Strategy 2:} Move all the zeros at $s_i$ to $0$, and all the ones at $s_{i+1}$ to $1$. The total cost is
        \[
            x\cdot s_i + y\cdot (1 - s_{i+1})
        =   (x + y) \cdot \left[(1 - p) \cdot s_i + p \cdot (1 - s_{i+1})\right].
        \]
    \end{itemize}

    In this case, $\Deltal_i = -x \cdot s_i$, $\Deltar_i = y\cdot (1 - s_{i+1})$, and the right-hand side of Inequality~\eqref{eq:cost-vs-biases} is given by
    \begin{align*}
        &~\left|x\cdot s_i - y\cdot (1 - s_{i+1})\right| + (s_{i+1} - s_i)\cdot\min\{x\cdot s_i, y\cdot (1 - s_{i+1})\}\\
    =   &~(x + y)\cdot\left[\left|(1 - p)\cdot s_i - p\cdot (1 - s_{i+1})\right| + (s_{i+1} - s_i)\cdot\min\{(1 - p)\cdot s_i, p\cdot (1 - s_{i+1})\}\right].
    \end{align*}

    We apply Lemma~\ref{lemma:technical-ineq-2} from Appendix~\ref{sec:approximation-proofs} with $\alpha = s_i$ and $\beta = s_{i+1}$ to show that the cost of the new transportation is at most
    \[
        10 \cdot \left[|\Deltal_i + \Deltar_i| + (s_{i+1} - s_i)\cdot\min\{|\Deltal_i|, |\Deltar_i|\}\right]
    \le 20\cdot\left(\costli_i + \costri_i\right).
    \]
\end{proof}

\section{Proof of the Upper Bound}\label{sec:upper}
We prove Theorem~\ref{thm:upper} in this section. We first note that it is sufficient to give an algorithm that achieves an $O(\sqrt{T})$ smooth calibration error, since this would imply the desired upper bound as follows:
\begin{align*}
    \Ex{}{\caldist(x, p)}
&\le\Ex{}{\lowercaldist(x, p)} + O(\sqrt{T}) \tag{Theorem~\ref{thm:lowercaldist-vs-caldist}} \\
&\le2\Ex{}{\smCE(x, p)} + O(\sqrt{T}) \tag{Lemma~\ref{lemma:smCE-vs-lowercaldist}}\\
&\le O(\sqrt{T}). 
\end{align*}

Our approach is based on a minimax argument similar to that of Hart~\cite{Hart22} for upper bounding the ECE in sequential calibration. Suppose we already know the adversary's strategy, which might be adaptive and randomized. At each step $t$, based on the previous outcomes $x_1, \ldots, x_{t-1}$ and predictions $p_1, \ldots, p_{t-1}$, we can calculate the conditional probability of $x_t = 1$. The natural strategy is then to predict this value exactly. Then, we may the view the sequences $x \in \{0, 1\}^{T}$ and $p \in [0, 1]^T$ as generated as below:
\begin{itemize}
    \item At each step $t$, $p_t$ is adversarially chosen based on $x_{1:(t-1)}$ and $p_{1:(t-1)}$.
    \item Then, we draw $x_t \sim \Bern(p_t)$.
\end{itemize}

Recall that the smooth calibration error $\smCE(x, p)$ is defined as
\[
    \sup_{f \in \F}\sum_{t=1}^{T}f(p_t)\cdot(x_t - p_t),
\]
where $\F$ is the family of $1$-Lipschitz functions from $[0, 1]$ to $[-1, 1]$. For each \emph{fixed} $f \in \F$, the random process $(X_0, X_1, \ldots, X_T)$ defined as $X_t \coloneqq \sum_{t'=1}^{t}f(p_{t'})\cdot(x_{t'} - p_{t'})$ is a martingale with bounded differences, so $X_T = \sum_{t=1}^{T}f(p_t)\cdot (x_t - p_t)$ is bounded by $O(\sqrt{T})$ with high probability. The difficulty, however, is to show that the same upper bound holds even if we take a supremum over all functions $f \in \F$.

\subsection{An Online Learning Setting}\label{sec:online-learning-setting}
Our proof is based on viewing the discussion above as an instance of online learning. In particular, we follow a formulation in~\cite{RST15a}.

An ``adversary'' and a ``player'' play a game with $T$ steps. At each step $t \in [T]$, the following happen in sequential order:
\begin{itemize}
    \item The adversary picks an ``instance'' $p^*_t \in [0, 1]$.
    \item The player, knowing $p^*_t$, commits to a distribution $\D_t$ over $[-1, 1]$, from which the predicted label $\hat y_t$ will be drawn.
    \item The adversary, knowing $\D_t$, generates the true label $y_t \in [-1, 1]$.
    \item The player draws prediction $\hat y_t \sim \D_t$ and incurs a loss of $\ell(\hat y_t, y_t)$.
\end{itemize}

The player's objective is to minimize the \emph{cumulative regret}, defined as the difference between the player's total loss and the total loss incurred by the best hypothesis in hindsight:
\[
    \Ex{}{\sum_{t=1}^{T}\ell(\hat y_t, y_t)}
-   \Ex{}{\inf_{f \in \F}\sum_{t=1}^{T}\ell(f(p^*_t), y_t)}.
\]
The adversary aims to maximize this regret. This setup exactly matches the learning setting defined in~\cite[Equation (10)]{RST15a}.

\subsection{Regret Bound and Sequential Rademacher Complexity}
The work of~\cite{RST15a} gives an upper bound on the optimal regret in the above online learning setting in terms of the \emph{sequential Rademacher complexity} of the function class $\F$.

\begin{definition}[Sequential Rademacher Complexity]
    The sequential Rademacher complexity of a family $\F$ of functions over $[0, 1]$ is
    \[
        \SRC_T(\F) \coloneqq \sup_{z_1, \ldots, z_T}\Ex{\sigma\sim\{\pm 1\}^T}{\sup_{f \in \F}\sum_{t=1}^{T}\sigma_tf(z_t(\sigma_1, \sigma_2, \ldots, \sigma_{t-1}))},
    \]
    where the outer supremum is taken over all $(z_1, \ldots, z_T)$ such that each $z_t: \{\pm1\}^{t-1} \to [0,1]$.
\end{definition}

\begin{theorem}[Theorem 8 of~\cite{RST15a}]\label{thm:RST-regret-bound}
    Suppose that for any $y \in [-1, 1]$, the loss function $\ell(\cdot, y)$ is convex and $L$-Lipschitz. Then, the optimal regret is at most $2L\cdot\SRC_T(\F)$.
\end{theorem}

Finally, we will use the following result that upper bounds $\SRC_T(\F)$ in terms of the covering numbers of $\F$.

\begin{theorem}[Theorem 4 of~\cite{RST15b}]\label{thm:RST-SRC-bound}
    Let $\F$ be a family of functions over $[0, 1]$. With respect to $z = (z_1, z_2, \ldots, z_T)$ where $z_t: \{\pm 1\}^{t-1} \to [0, 1]$, a family $\F'$ is a $\delta$-cover of $\F$ if, for any $\sigma \in \{\pm1\}^T$ and $f \in \F$, there exists $f' \in \F'$ such that
    \[
        \sqrt{\frac{1}{T}\sum_{t=1}^{T}(f(z_t(\sigma_{1:(t-1)})) - f'(z_t(\sigma_{1:(t-1)})))^2} \le \delta.
    \]
    Let $\N(\delta, \F, z)$ denote the size of the smallest $\delta$-cover of $\F$ with respect to $z$. Then,
    \[
        \SRC_T(\F) \le \sup_{z_1, z_2, \ldots, z_T}\inf_{\alpha \in [0, 1]}\left\{4\alpha T + 12\sqrt{T}\cdot\int_{\alpha}^{1}\sqrt{\log \N(\delta, \F, z)}~\rmd\delta\right\}.
    \]
\end{theorem}

We apply Theorem~\ref{thm:RST-SRC-bound} to upper bound the sequential Rademacher complexity of the class of Lipschitz functions.

\begin{lemma}\label{lemma:SRC-upper-bound}
    Let $\F$ be the family of $1$-Lipschitz functions from $[0, 1]$ to $[-1, 1]$. Then,
    \[
        \SRC_T(\F) = O(\sqrt{T}).
    \]
\end{lemma}

The lemma is proved by a standard construction of covers for the class of Lipschitz functions.

\begin{proof}
    By Theorem~\ref{thm:RST-SRC-bound}, $\SRC_T(\F)$ is upper bounded by
    \begin{equation}\label{eq:SRC-upper-bound}
        \sup_{z_1, z_2, \ldots, z_T}\inf_{\alpha \in [0, 1]}\left\{4\alpha T + 12\sqrt{T}\cdot\int_{\alpha}^{1}\sqrt{\log \N(\delta, \F, z)}~\rmd\delta\right\}.
    \end{equation}

    We fix $z = (z_1, \ldots, z_T)$ and $\delta \in (0, 1]$, and give an upper bound on $\N(\delta, \F, z)$. Let $k = \lceil 2/\delta \rceil$. For a function $f \in \F$, we construct another function $\hat f$ that takes value $\frac{\lfloor f(i/k)\cdot k\rfloor}{k}$ at $i/k$ for each $i \in \{0, 1, \ldots, k\}$. On each interval $[(i-1)/k, i/k]$, $\hat f$ is the linear interpolation between $f((i-1)/k)$ and $f(i/k)$.

    We first note that $\hat f$ is $1$-Lipschitz. For each $i \in [k]$, since $f$ is $1$-Lipschitz, we have
    \[
        \left|f\left(\frac{i-1}{k}\right)\cdot k - f\left(\frac{i}{k}\right)\cdot k\right| = k \cdot \left|f\left(\frac{i-1}{k}\right) - f\left(\frac{i}{k}\right)\right| \le k\cdot\frac{1}{k} = 1.
    \]
    It follows that $\lfloor f((i-1)/k)\cdot k\rfloor$ and $\lfloor f(i/k)\cdot k\rfloor$ differ by at most $1$, which implies $|\hat f((i-1)/k) - \hat f(i/k)| \le 1/k$. Thus, the linear interpolation on the interval $[(i-1)/k, i/k]$ has a slope between $\pm 1$. This shows that $\hat f$ is $1$-Lipschitz.

    Then, we argue that $\hat f$ is point-wise close to $f$. For each $i \in \{0, 1, \ldots, k\}$, we have
    \[
        \left|\hat f(i/k) - f(i/k)\right| \le \frac{1}{k} \le \frac{\delta}{2}.
    \]
    For general $x \in [0, 1]$, there exists $i \in \{0, 1, \ldots, k\}$ such that $|x - i/k| \le 1/(2k)$. It follows that
    \begin{align*}
        |\hat f(x) - f(x)|
    &\le|\hat f(x) - \hat f(i/k)| + |\hat f(i/k) - f(i/k)| + |f(i/k) - f(x)|\\
    &\le |x - i/k| + \frac{\delta}{2} + |x - i/k| \tag{$f$ and $\hat f$ are $1$-Lipschitz}\\
    &\le \frac{1}{2k} + \frac{\delta}{2} + \frac{1}{2k}
    \le \delta.
    \end{align*}
    In particular, regardless of the value of $\sigma \in \{\pm1\}^T$, we have
    \[
        \sqrt{\frac{1}{T}\sum_{t=1}^{T}(f(z_t(\sigma_{1:(t-1)})) - \hat f(z_t(\sigma_{1:(t-1)})))^2}
    \le \sqrt{\frac{1}{T}\cdot T\delta^2}
    =   \delta.
    \]
    
    Then, we show that the resulting function $\hat f$ falls into a small set. $\hat f$ is uniquely determined by its value on $\{0, 1/k, 2/k, \ldots, 1\}$. $\hat f(0)$ is one of the $2k + 1$ multiples of $1/k$ in $[-1, 1]$. For each $i \in [k]$, $\hat f(i/k) - \hat f((i-1)/k)$ falls into $\{-1/k, 0, 1/k\}$. Therefore, $\hat f$ falls into a set of size at most
    \[
        (2k + 1)\cdot 3^k
    \le 3^{2k}
    \le 3^{6/\delta}.
    \]
    This gives $\N(\delta, \F, p) \le 3^{6/\delta}$.
    
    Finally, picking $\alpha = 0$ in the expression in~\eqref{eq:SRC-upper-bound} shows that
    \[
        \SRC(\F) \le 12\sqrt{T}\int_{0}^{1}\sqrt{\frac{6\log 3}{\delta}}~\rmd\delta = O(\sqrt{T}).
    \]
\end{proof}

\subsection{Proof of Theorem~\ref{thm:upper}}
Now we proceed to the proof. A technical issue with the discussion earlier in this section is that, to apply the minimax theorem, the action spaces of the two players need to be finite. This is not true since the forecaster is allowed to make arbitrary predictions between $0$ and $1$. To deal with this issue, we will force the forecaster to restrict its predictions to a $1/T$-net of $[0, 1]$. Since the smooth calibration error $\smCE(x, p)$ is continuous in $p$, this rounding does not blow up the error by much.

\begin{proof}[Proof of Theorem~\ref{thm:upper}]
We will show that, even if the forecaster is only allowed to predict the values in $P \coloneqq \{0, 1/T, 2/T, \ldots, 1\}$, it is still possible to achieve an $O(\sqrt{T})$ smooth calibration error. By Theorem~\ref{thm:lowercaldist-vs-caldist} and Lemma~\ref{lemma:smCE-vs-lowercaldist}, this would give the desired $O(\sqrt{T})$ bound on the calibration distance.

After restricting the space of predictions, a \emph{deterministic} strategy of the adversary (resp., the forecaster) is simply a function from $\bigcup_{t=1}^{T}(\{0, 1\}\times P)^{t-1}$ to $\{0, 1\}$ (resp., to $P$). Both sets are finite (albeit of a doubly exponential size). In general, both players may play a mixture of deterministic strategies. By von Neumann's minimax theorem, it suffices to prove that for every fixed (possibly mixed) strategy of the adversary, the forecaster can achieve an $O(\sqrt{T})$ smooth calibration error.

\paragraph{The forecaster's algorithm.} Now, we describe one such strategy for the forecaster:
\begin{itemize}
    \item At each step $t \in [T]$, based on $x_{1:(t-1)}$ and $p_{1:(t-1)}$, compute the conditional probability for the adversary to play $x_t = 1$. Let $p^*_t$ denote this value.
    \item Predict $p_t \coloneqq \frac{\lfloor T \cdot p^*_t\rfloor}{T}$, which is $p^*_t$ rounded down to the nearest value in $P$.
\end{itemize}
We note that it is, in turn, sufficient to upper bound the expected value of $\smCE(p^*, x)$. This is because, by Lemma~\ref{lemma:Lipschitz-continuity} from Appendix~\ref{sec:preliminaries-proofs},
\[
    \Ex{}{\smCE(x, p)}
\le \Ex{}{\smCE(x, p^*)} + 2\Ex{}{\|p - p^*\|_1},
\]
whereas by our choice of $p$, $\|p - p^*\|_1$ is always at most $\frac{1}{T}\cdot T = 1$.

After fixing the forecaster's strategy, the ``game'' between the adversary and the forecaster can be equivalently described as the following procedure. At the beginning, a function $g$ from $\bigcup_{t=1}^{T}(\{0, 1\}^{t-1}\times [0, 1]^{t-1})$ to $[0, 1]$ is adversarially chosen. Then, for $t = 1, 2, \ldots, T$:
\begin{itemize}
    \item Pick $p^*_t = g\left(x_{1:(t-1)}, p^*_{1:(t-1)}\right)$.
    \item Draw $x_t$ from $\Bern(p^*_t)$.
\end{itemize}
Note that the first step above is equivalent to the original game, since the predictions $p_{1:(t-1)}$ are determined by $p^*_{1:(t-1)}$.

\paragraph{Reduction to online learning.} Now, we further rephrase the procedure described above as an online learning setup in Section~\ref{sec:online-learning-setting}. At each step $t \in [T]$, the adversary picks the ``instance'' $p^*_t$ as $g\left(x_{1:(t-1)}, p^*_{1:(t-1)}\right)$. The player then commits to an arbitrary distribution $\D_t$ over $[-1, 1]$. (We will show later that the choice of $\D_t$ is inconsequential.) The adversary picks the true label $y_t$ by drawing $x_t \sim \Bern(p^*_t)$ and setting $y_t = x_t - p^*_t$. Finally, the player draws $\hat y_t \sim \D_t$ and incurs a loss of $\ell(\hat y_t, y_t) \coloneqq \hat y_t \cdot y_t$.

The regret in the above setup can be simplified into
\begin{equation}\label{eq:regret-simplified}
    \Ex{}{\sum_{t=1}^{T}\hat y_t\cdot(x_t - p^*_t)}
-   \Ex{}{\inf_{f \in \F}\sum_{t=1}^{T}f(p^*_t)\cdot (x_t - p^*_t)}.
\end{equation}
By the definition of the learning procedure and the choice of $x_t$, we have
\[
    \Ex{}{\hat y_t\cdot(x_t - p^*_t)}
=   \Ex{p^*_t, \hat y_t}{\hat y_t\cdot(\Ex{}{x_t|p^*_t, \hat y_t} - p^*_t)}
=   \Ex{p^*_t, \hat y_t}{\hat y_t\cdot(p^*_t - p^*_t)}
=   0
\]
for every $t \in [T]$, so the first term in~\eqref{eq:regret-simplified} is always $0$ regardless of how the player picks $\D_t$ (which determines $\hat y_t$). Note that $f \in \F$ if and only if $-f \in \F$, so the expression in~\eqref{eq:regret-simplified} is equal to
\[
    -\Ex{}{\inf_{f \in \F}\sum_{t=1}^{T}f(p^*_t)\cdot(x_t - p^*_t)}
=   \Ex{}{\sup_{f \in \F}\sum_{t=1}^{T}f(p^*_t)\cdot(x_t - p^*_t)} = \Ex{}{\smCE(x, p^*)}.
\]

So far, we proved that in our online learning setting, when the adversary plays a specific strategy (namely, pick the instance $p^*_t$ as $g(x_{1:(t-1)}, p_{1:(t-1)})$ and the true label $y_t$ as $x_t - p^*_t$, where $x_t \sim \Bern(p^*)$), the regret of any player is given by $\Ex{}{\smCE(x, p^*)}$. Therefore, $\Ex{}{\smCE(x, p^*)}$ is upper bounded by the optimal regret for this setup. Note that for any $y_t \in [-1, 1]$, the loss function $\ell(\cdot, y_t)$ is convex and $1$-Lipschitz in the first parameter. By Theorem~\ref{thm:RST-regret-bound} and Lemma~\ref{lemma:SRC-upper-bound}, the optimal regret is at most $2\SRC(\F) = O(\sqrt{T})$. This concludes the proof.
\end{proof}

\section{Improved Forecasters for Random Bits}\label{sec:random-bits}
In this section, we prove Proposition~\ref{prop:random-bits-polylog}, which gives a $\polylog(T)$ calibration distance when the adversary plays $T$ independent random bits. We will first present a simple forecasting algorithm with an $O(T^{1/3})$ calibration distance in expectation. Then, we use the same idea to further improve the calibration distance to $O(\log^{3/2} T)$.

\subsection{A Sub-Square-Root Upper Bound for Random Bits}
Algorithm~\ref{algo:random-bits} gives a forecasting strategy that achieves an $O(T^{1/3})$ smooth calibration error on a sequence of $T$ random bits.

\begin{algorithm2e}
    \caption{Fixed-Bias Forecaster for Random Bits}
    \label{algo:random-bits}
    \KwIn{Time horizon $T$. Parameter $\eps > 0$. Online access to $x_1, x_2, \ldots, x_T$.}
    \KwOut{Predictions $p_1, p_2, \ldots, p_T$.}
    $S_0 \gets 0$\;
    \For{$t \in [T]$} {
        $p_t \gets 1/2 + \eps\cdot\sgn(S_{t-1})$\;
        Predict $p_t$\;
        Observe $x_t$\;
        $S_t \gets S_{t-1} + (x_t - p_t)$\;
    }
\end{algorithm2e}

The algorithm keeps track of $S_t = \sum_{t'=1}^{t}(x_{t'} - p_{t'})$, the difference between the total outcomes and the total predictions in the first $t$ steps. If $S_{t-1} > 0$ (resp., $S_{t-1} < 0$), the forecaster predicts a value slightly higher (resp., lower) than $1/2$, in the hope that $S_t$ will get closer to $0$.

Now we analyze the deviation from calibration incurred by the above algorithm. We will start by upper bounding the smooth calibration error of Algorithm~\ref{algo:random-bits}, and then invoke Theorem~\ref{thm:lowercaldist-vs-caldist} and Lemma~\ref{lemma:smCE-vs-lowercaldist} to get an upper bound on the calibration distance.

We start with the following simple upper bound on the smooth calibration error. Recall the definition of $\smCE$ and $\Delta_{\alpha}$ from Section~\ref{sec:preliminaries}.

\begin{lemma}\label{lemma:smCE-bound}
    \[
        \smCE(x, p)
    \le \left|\sum_{\alpha \in [0, 1]}\Delta_{\alpha}\right| + \sum_{\alpha \in [0, 1]}|\alpha - 1/2|\cdot|\Delta_{\alpha}|.
    \]
\end{lemma}

\begin{proof}
    By definition, we have
    \begin{align*}
        \smCE(x, p)
    &=  \sup_{f \in \F}\sum_{\alpha \in [0, 1]}f(\alpha)\cdot\Delta_{\alpha}\\
    &=  \sup_{f \in \F}\left[f(1/2) \cdot \sum_{\alpha \in [0, 1]}\Delta_{\alpha} + \sum_{\alpha \in [0, 1]}(f(\alpha) - f(1/2)) \cdot \Delta_{\alpha}\right]\\
    &\le\sup_{f \in \F}\left[f(1/2) \cdot \sum_{\alpha \in [0, 1]}\Delta_{\alpha}\right] + \sum_{\alpha \in [0, 1]}\sup_{f \in \F}\left[(f(\alpha) - f(1/2)) \cdot \Delta_{\alpha}\right]\\
    &\le\left|\sum_{\alpha \in [0, 1]}\Delta_{\alpha}\right| + \sum_{\alpha \in [0, 1]}|\alpha - 1/2| \cdot |\Delta_{\alpha}|.
    \end{align*}
    The last step holds since the functions in $\F$ are both bounded and $1$-Lipschitz.
\end{proof}

Note that the $\sum_{\alpha \in [0, 1]}\Delta_{\alpha}$ term in Lemma~\ref{lemma:smCE-bound} is exactly $S_T = \sum_{t=1}^{T}(x_t - p_t)$ in Algorithm~\ref{algo:random-bits}. The following lemma gives a bound on the stochastic process $(S_0, S_1, \ldots, S_T)$.

\begin{lemma}\label{lemma:random-walk-with-drifts}
    For $\eps \in (0, \frac{1}{2}]$, consider the stochastic process $(X_0, X_1, X_2, \ldots)$ defined as follows:
    \begin{itemize}
        \item $X_0 = 0$.
        \item $x_1, x_2, x_3, \ldots$ are independent samples from $\Bern(1/2)$.
        \item For $t \ge 1$, $X_t = X_{t-1} + x_t - \left(\frac{1}{2} + \eps\cdot\sgn(X_{t-1})\right)$.
    \end{itemize}
    Then, for any $t \ge 0$, $\Delta \ge 0$ and $C = e^{1/2}$, it holds that
    \[
        \pr{}{|X_t| \ge \Delta} \le C\cdot e^{-\eps\Delta}.
    \]
\end{lemma}

\begin{proof}
    We prove the lemma by an induction on $t$. The inequality clearly holds for $t = 0$ and all $\Delta \ge 0$. Now, assuming the inequality for $X_{t - 1}$ and all $\Delta \ge 0$, we prove the $X_t$ case. When $\Delta \le 1$, the inequality holds trivially, since we have $\eps\Delta \le 1/2$, which implies
    \[
        \pr{}{|X_t| \ge \Delta}
    \le 1
    =   C\cdot e^{-1/2}
    \le C\cdot e^{-\eps\Delta}.
    \]

    It remains to handle the $\Delta > 1$ case. In order to reach $|X_t| \ge \Delta > 1$, we must have $X_{t-1} \ne 0$; otherwise we would have $X_t \in \{-1/2, 1/2\}$. Furthermore, one of the following two must hold:
    \begin{itemize}
        \item $|X_{t-1}| \ge \Delta - (1/2 - \eps)$ and $\sgn(x_t - 1/2) = \sgn(X_{t-1})$.
        \item $|X_{t-1}| \ge \Delta + (1/2 + \eps)$ and $\sgn(x_t - 1/2) = -\sgn(X_{t-1})$.
    \end{itemize}
    Note that by the inductive hypothesis, $|X_{t-1}| \ge \Delta - (1/2 - \eps)$ holds with probability at most $C\cdot e^{-\eps[\Delta - (1/2 - \eps)]}$. In addition, conditioning on this event, the probability of $\sgn(x_t - 1/2) = \sgn(X_{t-1})$ is still $1/2$ by independence. Thus, the probability of the former is at most $\frac{C}{2} e^{-\eps\Delta + \eps/2 - \eps^2}$. An analogous argument upper bounds the probability of the latter condition by $\frac{C}{2}e^{-\eps\Delta - \eps/2 - \eps^2}$.

    To conclude the inductive step, we need the inequality
    \[
        \frac{C}{2} e^{-\eps\Delta + \eps/2 - \eps^2} + \frac{C}{2}e^{-\eps\Delta - \eps/2 - \eps^2}
    \le C\cdot e^{-\eps\Delta},
    \]
    which is equivalent to
    \[
        e^{\eps/2} + e^{-\eps/2} \le 2e^{\eps^2}.
    \]
    The last inequality can be shown to hold for all $\eps \ge 0$ via Taylor expansion. This completes the proof.
\end{proof}

\begin{lemma}\label{lemma:random-bits-smCE}
    On a sequence of $T$ independent random bits, the smooth calibration error incurred by Algorithm~\ref{algo:random-bits} with $\eps = T^{-1/3}$ is $O(T^{1/3})$ in expectation.
\end{lemma}

\begin{proof}
    Note that Algorithm~\ref{algo:random-bits} only predicts three different values: $1/2$, $1/2 + \eps$, and $1/2 - \eps$. In light of Lemma~\ref{lemma:smCE-bound}, it suffices to upper bound the expectation of the following three terms at time $T$: (1) $|\Delta_{1/2} + \Delta_{1/2 + \eps} + \Delta_{1/2 - \eps}|$; (2) $|\Delta_{1/2 + \eps}|$; (3) $|\Delta_{1/2 - \eps}|$.

    \paragraph{The first term.} The first part is done by Lemma~\ref{lemma:random-walk-with-drifts}: The stochastic process $S_t \coloneqq \sum_{t'=1}^{t}(x_{t'} - p_{t'})$ exactly matches the one defined in Lemma~\ref{lemma:random-walk-with-drifts}. Therefore, the first term is exactly the absolute value of $S_T = \sum_{t=1}^{T}(x_{t} - p_{t})$. By Lemma~\ref{lemma:random-walk-with-drifts}, we have
    \[
        \Ex{}{|S_T|}
    =   \int_{0}^{+\infty}\pr{}{|S_T| \ge \tau}~\rmd\tau
    \le e^{1/2}\int_{0}^{+\infty}e^{-\eps\tau}~\rmd\tau
    =   \frac{e^{1/2}}{\eps}
    =   O(T^{1/3}).
    \]

    \paragraph{The second term.} To analyze the second part, it is convenient to assume that the nature samples random bits $b_1, b_2, \ldots, b_T \sim \Bern(1/2)$ independently at the beginning, and uses these bits one by one as the outcomes for the steps on which $1/2 + \eps$ is predicted. More formally, whenever the forecaster predicts $1/2 + \eps$ at time $t$, the nature calculates $k = \sum_{t'=1}^{t}\1{p_{t'} = 1/2 + \eps}$ and sets $x_t = b_k$. Note that this change does not alter the distribution of the random outcomes, and thus the execution of Algorithm~\ref{algo:random-bits} remains unchanged.

    Then, we note that $\Delta_{1/2 + \eps}$ can be written as
    \[
        \sum_{i=1}^{m}b_i - m \cdot (1/2 + \eps),
    \]
    where $m$ is the number of times $1/2 + \eps$ is predicted. By a Chernoff bound and a union bound over $m \in [T]$, for any $\delta \in (0, 1)$, it holds with probability $1 - \delta$ that for all $m \in [T]$,
    \[
        \left|\sum_{i=1}^{m}b_i - \frac{m}{2}\right| \le \sqrt{\frac{T\ln(2T/\delta)}{2}}.
    \]
    The above implies
    \[
        |\Delta_{1/2 + \eps}|
    \le \left|\sum_{i=1}^{m}b_i - \frac{m}{2}\right| + m\eps
    \le \sqrt{\frac{T\ln(2T/\delta)}{2}} + T\eps.
    \]
    Setting $\delta = 1/T$ and $\eps = 1/T^{1/3}$ shows that $|\Delta_{1/2 + \eps}| = O(T^{2/3})$ with probability $1 - 1/T$. Finally, since $|\Delta_{1/2 + \eps}|$ is always upper bounded by $T$, we have
    \[
        \Ex{}{|\Delta_{1/2 + \eps}|} = O(T^{2/3}) + \delta\cdot T = O(T^{2/3}).
    \]

    \paragraph{Wrapping up.} By an analogous argument to the above, we have $\Ex{}{|\Delta_{1/2 - \eps}|} = O(T^{2/3})$. Finally, by Lemma~\ref{lemma:smCE-bound}, the expected smooth calibration error is upper bounded by
    \[
        \Ex{}{|S_T|} + \eps\Ex{}{|\Delta_{1/2 + \eps}|} + \eps\Ex{}{|\Delta_{1/2 - \eps}|}
    \le O(T^{1/3}) + 2T^{-1/3}\cdot O(T^{2/3})
    =   O(T^{1/3}).
    \]
\end{proof}

\begin{corollary}\label{cor:random-bits-caldist}
    The calibration distance incurred by Algorithm~\ref{algo:random-bits} with $\eps = T^{-1/3}$ is $O(T^{1/3})$ in expectation.
\end{corollary}

\begin{proof}
    Lemmas \ref{lemma:smCE-vs-lowercaldist}~and~\ref{lemma:random-bits-smCE} imply that Algorithm~\ref{algo:random-bits} incurs an $O(T^{1/3})$ lower calibration distance in expectation. Since Algorithm~\ref{algo:random-bits} predicts at most three different values, the corollary follows from the second part of Theorem~\ref{thm:lowercaldist-vs-caldist}.
\end{proof}

\subsection{A Polylogarithmic Calibration Distance for Random Bits}

In the previous section, we saw how a simple strategy improves the calibration distance from $\Theta(T^{1/2})$ to $O(T^{1/3})$ on a random bit sequence of length $T$. It turns out that applying the same idea in a slightly more involved way would reduce the distance significantly, to $O(\log^{3/2}T)$.

The forecaster's strategy starts by predicting the value $1/2$ for $T/2$ steps. After that, we expect an $O(\sqrt{T})$ gap between the counts of ones and zeros so far. Say that the number of ones is larger. Then, in the remaining $T/2$ steps, the forecaster keeps predicting $1/2 + \sqrt{\ln T / T}$, until the sum of $x_t$'s and the sum of $p_t$'s are roughly the same. The key observation is that, when the forecaster succeeds in bringing this difference down to zero, the expected smooth calibration error so far is merely $O(\sqrt{\log T})$. For the remaining time steps, we recursively apply the same strategy (with a smaller value of $T$), and the process must end in $O(\log T)$ rounds. It is relatively easy to show that the errors in different rounds can be aggregated together to give the $O(\log^{3/2}T)$ upper bound.

Formally, we state the strategy of the forecaster in Algorithm~\ref{algo:random-bits-log}.

\begin{algorithm2e}
    \caption{Adaptive-Bias Forecaster for Random Bits}
    \label{algo:random-bits-log}
    \KwIn{Time horizon $T$. Parameter $\eps > 0$. Online access to $x_1, x_2, \ldots, x_T$.}
    \KwOut{Predictions $p_1, p_2, \ldots, p_T$.}
    $t \gets 0$\;
    \For{$r = 1, 2, 3, \ldots$} {
        $T^{(r)} \gets T - t$\;
        \For{$i = 1, 2, \ldots, T^{(r)} / 2$} {
            $t \gets t + 1$\;
            $p_t \gets 1/2$;
            $p^{(r)}_i \gets 1/2$;
            Predict $p_t$\;
            Observe $x_t$;
            $x^{(r)}_i \gets x_t$\;
        }
        $\Delta^{(r)}\gets \sum_{i=1}^{T^{(r)} / 2}(x^{(r)}_i - p^{(r)}_i)$\;\label{line:Delta_r}
        $\eps^{(r)} \gets \sgn\left(\Delta^{(r)}\right)\cdot\min\left\{\frac{2\left|\Delta^{(r)}\right|}{T^{(r)}} + \sqrt{\frac{\ln T^{(r)}}{T^{(r)}}}, \frac{1}{2}\right\}$\;\label{line:eps_r}
        \For{$i = T^{(r)} / 2 + 1, \ldots, T^{(r)}$}{
            $t \gets t + 1$\;
            $p_t \gets 1/2+\eps^{(r)}$;
            $p^{(r)}_i \gets 1/2+\eps^{(r)}$;
            Predict $p_t$\;
            Observe $x_t$;
            $x^{(r)}_i \gets x_t$\;
            \lIf{$\sum_{j=1}^{i}\left(x^{(r)}_j - p^{(r)}_j\right) \in [-1, 1]$}{\textbf{break}}\label{line:break}
        }
        \lIf{$t = T$}{\textbf{break}}
    }
\end{algorithm2e}

We will prove the following bound, which immediately implies Proposition~\ref{prop:random-bits-polylog}.

\begin{lemma}\label{lemma:random-bits-polylog-smCE}
    On a sequence of $T$ independent random bits, the smooth calibration error incurred by Algorithm~\ref{algo:random-bits-log} is $O(\log^{3/2}T)$ in expectation.
\end{lemma}

\begin{proof}[Proof of Proposition~\ref{prop:random-bits-polylog}]
    Note that in Algorithm~\ref{algo:random-bits-log}, the outer for-loop is executed at most $O(\log T)$ times. Furthermore, each prediction made by the forecaster is either $1/2$, or a value  $p_i^{(r)}$ specific to a round $r$. Therefore, $p_1, p_2, \ldots, p_T$ contain at most $O(\log T)$ different values. Then, by Lemma~\ref{lemma:smCE-vs-lowercaldist}, Lemma~\ref{lemma:random-bits-polylog-smCE} and the second part of Theorem~\ref{thm:lowercaldist-vs-caldist}, Algorithm~\ref{algo:random-bits-log} achieves
    \[
        \Ex{}{\caldist(x,p)}
    \le O(1)\cdot\Ex{}{\smCE(x, p)} + O(\log T)
    =   O(\log^{3/2}T).
    \]
\end{proof}

Our proof Lemma~\ref{lemma:random-bits-polylog-smCE} is decomposed into two parts: First, we argue that it suffices to bound the expected smooth calibration error on outcomes $x^{(r)}$ and predictions $p^{(r)}$, and their sum gives an upper bound on $\smCE(x, p)$. Second, we show that for every $r$, the error is bounded by $O(\sqrt{\log T})$ in expectation. Lemma~\ref{lemma:random-bits-polylog-smCE} then directly follows, since there are at most $O(\log T)$ rounds.

Formally, we have the following two lemmas. The first lemma simply states that the smooth calibration error is subadditive with respect to sequence concatenation.

\begin{lemma}\label{lemma:smCE-decomp}
    Let $x^{(1)} \in \{0, 1\}^{\len^{(1)}}$, $\ldots$, $x^{(R)} \in \{0, 1\}^{\len^{(R)}}$ be binary sequences, and $p^{(1)} \in [0, 1]^{\len^{(1)}}$, $\ldots$, $p^{(R)} \in [0, 1]^{\len^{(R)}}$ be sequences with the corresponding lengths. Let $x$ and $p$ be the concatenations of $x^{(r)}$ and $p^{(r)}$ in ascending order. Then,
    \[
        \smCE(x, p) \le \sum_{r=1}^{R}\smCE(x^{(r)}, p^{(r)}).
    \]
\end{lemma}

The second lemma bounds the expected smooth calibration error in each round.

\begin{lemma}\label{lemma:smCE-single-round}
    Let $R \coloneqq \lceil \log_2T\rceil$. Over the randomness in the execution of Algorithm~\ref{algo:random-bits-log}, define random variables $X_1, X_2, \ldots, X_R$ as follows: For every $r \in [R]$, if the algorithm reaches the $r$-th round, $X_r = \smCE(x^{(r)}, p^{(r)})$; otherwise, $X_r = 0$. Then, for every $r \in [R]$,
    \[
        \Ex{}{X_r} = O(\sqrt{\log T}),
    \]
    where $O(\cdot)$ hides a universal constant that is independent of $T$ and $r$.
\end{lemma}

It is clear that Lemma~\ref{lemma:random-bits-polylog-smCE} is a directly corollary of Lemmas \ref{lemma:smCE-decomp}~and~\ref{lemma:smCE-single-round}.

\begin{proof}[Proof of Lemma~\ref{lemma:smCE-decomp}]
    By definition of the smooth calibration error, we have
    \begin{align*}
        \smCE(x,p)
    &=   \sup_{f \in \F}\left[\sum_{r=1}^R\sum_{t=1}^{\len^{(r)}}f(p^{(r)}_t)\cdot(x^{(r)}_t - p^{(r)}_t)\right]\\
    &\le\sum_{r=1}^R\sup_{f \in \F}\left[\sum_{t=1}^{\len^{(r)}}f(p^{(r)}_t)\cdot(x^{(r)}_t - p^{(r)}_t)\right]\\
    &=   \sum_{r=1}^R\smCE(x^{(r)}, p^{(r)}).
    \end{align*}
\end{proof}

\begin{proof}[Proof of Lemma~\ref{lemma:smCE-single-round}]
    We fix a round $r \in [R]$, and condition on the event that $T^{(r)} = L$ holds at the beginning of the $r$-th iteration of the outer for-loop in Algorithm~\ref{algo:random-bits-log}. Note that the event $T^{(r)} = L$ only depends on the first $T - L$ bits $x_1, x_2, \ldots, x_{T-L}$, so conditioning on $T^{(r)} = L$ does not change the distribution of the remaining random bits $x_{T-L+1}$ through $x_T$.

    We say that Round~$r$ of Algorithm~\ref{algo:random-bits-log} \emph{succeeds} if the round ends by taking the break on Line~\ref{line:break} in the inner for-loop; Round~$r$ \emph{fails} otherwise.

    \paragraph{The case that Round~$r$ succeeds.} We start by upper bounding the value of $X_r$ assuming that the $r$-th round succeeds. Let $\len^{(r)}$ denote the number of steps in Round~$r$. For $\alpha \in \{1/2, 1/2 + \eps^{(r)}\}$, we define $\Delta_{\alpha}$ as the total bias incurred by the steps on which $\alpha$ is predicted during Round~$r$, i.e.,
    \[
        \Delta_{\alpha} \coloneqq \sum_{i=1}^{\len^{(r)}}\left(x^{(r)}_i - p^{(r)}_i\right) \cdot \1{p^{(r)}_i = \alpha}.
    \]
    By definition of the smooth calibration error, we have
    \begin{align*}
        X_r
    &=   \smCE\left(x^{(r)}, p^{(r)}\right)
    =   \sup_{f \in \F}\left[f(1/2)\Delta_{1/2} + f(1/2 + \eps^{(r)})\Delta_{1/2 + \eps^{(r)}}\right]\\
    &=  \sup_{f \in \F}\left[f(1/2)(\Delta_{1/2} + \Delta_{1/2 + \eps^{(r)}}) + (f(1/2 + \eps^{(r)}) - f(1/2))\Delta_{1/2 + \eps^{(r)}}\right]\\
    &\le    \sup_{f \in \F}\left[f(1/2)(\Delta_{1/2} + \Delta_{1/2 + \eps^{(r)}})\right] + \sup_{f \in \F}\left[(f(1/2 + \eps^{(r)}) - f(1/2))\Delta_{1/2 + \eps^{(r)}}\right]\\
    &\le   \left|\Delta_{1/2} + \Delta_{1/2 + \eps^{(r)}}\right| + \left|\eps^{(r)}\right|\cdot\left|\Delta_{1/2 + \eps^{(r)}}\right|,
    \end{align*}
    where the last step follows since every $f \in \F$ is $1$-Lipschitz and bounded between $-1$ and $+1$.
    
    When Round~$r$ succeeds, we have
    \[
        \Delta_{1/2} + \Delta_{1/2 + \eps^{(r)}}
    =   \sum_{i=1}^{\len^{(r)}}\left(x^{(r)}_i - p^{(r)}_i\right)
    \in [-1, 1],
    \]
    which further implies $\left|\Delta_{1/2 + \eps^{(r)}}\right| \le |\Delta_{1/2}| + 1$.

    Therefore, assuming the success of Round~$r$, we have
    \[
        X_r
    \le 1 + \left|\eps^{(r)}\right|\cdot\left(\left|\Delta_{1/2}\right| + 1\right)
    \le \frac{3}{2} + \frac{2\Delta_{1/2}^2}{L} + |\Delta_{1/2}|\sqrt{\frac{\ln L}{L}}.
    \]

    \paragraph{Control the failure probability.} To show that Round~$r$ succeeds with high probability, we first argue that when Algorithm~\ref{algo:random-bits-log} chooses $\eps^{(r)}$ in Line~\ref{line:eps_r}, the minimum takes the first value with high probability. Indeed, this is true as long as
    \[
        \frac{2|\Delta^{(r)}|}{L} + \sqrt{\frac{\ln L}{L}} \le \frac{1}{2},
    \]
    which is equivalent to
    \[
        |\Delta^{(r)}|
    \le \left(\frac{1}{4} - \frac{1}{2}\sqrt{\frac{\ln L}{L}}\right)\cdot L.
    \]
    For sufficiently large $L$, we have $\frac{1}{2}\sqrt{\ln L / L} \le 1/8$, so the above is true as long as $|\Delta^{(r)}| \le L / 8$. Noting that $\Delta^{(r)}$ is the difference between a sample from $\Binomial(L/2, 1/2)$ and its mean $L/4$, it follows from a Chernoff bound that $|\Delta^{(r)}| \le L / 8$ holds with probability $1 - e^{-\Omega(L)}$.

    Now, assuming that $\eps^{(r)}$ satisfies the aforementioned condition, what is the probability for Round~$r$ to fail? Without loss of generality, assume that $\Delta^{(r)} \ge 0$. Then, the failure of Round~$r$ would imply $\Delta^{(r)} + \sum_{i=L/2+1}^{L}(x^{(r)}_i - p^{(r)}_i) \ge 0$; otherwise there must be a time step $i \in [L/2+1, L]$ at which $\sum_{j=1}^{i}(x^{(r)}_j - p^{(r)}_j)$ falls into the interval $[-1, 1]$, which allows the round to end. Recalling that $p^{(r)}_i = 1/2 + \eps^{(r)}$ for every $i \in [L/2+1, L]$, we can rewrite the inequality $\Delta^{(r)} + \sum_{i=L/2+1}^{L}(x^{(r)}_i - p^{(r)}_i) \ge 0$ as
    \[
        \sum_{i=L/2+1}^{L}x^{(r)}_i - \frac{L}{4}
    \ge -\Delta^{(r)} + \frac{L}{2}\eps^{(r)}
    =   \frac{1}{2}\sqrt{L\ln L},
    \]
    where the second step applies $\eps^{(r)} = 2\Delta^{(r)}/L + \sqrt{\ln L / L}$. Again, since the left-hand side above is a Binomial random variable (from $\Binomial(L/2, 1/2)$) minus its mean, the probability for the above inequality to hold is, by a Chernoff bound, at most
    \[
        \exp\left(-2\cdot\frac{L}{2}\cdot\left(\frac{\sqrt{L\ln L}/2}{L/2}\right)^2\right) = \frac{1}{L}.
    \]
    Therefore, we conclude that the probability for Round~$r$ to fail (conditioning on $T^{(r)} = L$) is at most $e^{-\Omega(L)} + \frac{1}{L} = O(1/L)$.

    \paragraph{Put everything together.} Our analysis for the case that Round~$r$ succeeds, together with the observation that $X_r$ is always at most $L$, implies that when $T^{(r)} = L$,
    \[
        X_r \le \frac{3}{2} + \frac{2\Delta_{1/2}^2}{L} + |\Delta_{1/2}|\sqrt{\frac{\ln L}{L}} + L \cdot \1{\text{Round }r\text{ fails}}
    \]
    always holds. Therefore, we have
    \begin{align*}
        \Ex{}{X_r|T^{(r)} = L}
    &\le\Ex{}{\frac{3}{2} + \frac{2\Delta_{1/2}^2}{L} + |\Delta_{1/2}|\sqrt{\frac{\ln L}{L}}\Bigg|T^{(r)} = L} + L\cdot\pr{}{\text{Round }r\text{ fails}\Big|T^{(r)} = L}\\
    &\le O(\sqrt{\log L}) + L \cdot O(1/L)\\
    &=  O(\sqrt{\log T}). \tag{$L \le T$}
    \end{align*}
    The second step above applies the observation that conditioning on $T^{(r)} = L$, $\Delta_{1/2}$ is the difference between a sample from $\Binomial(L/2, 1/2)$ and its mean $L/4$, which implies $\Ex{}{\Delta^2_{1/2}} = O(L)$ and $\Ex{}{|\Delta_{1/2}|} = O(\sqrt{L})$.

    Finally, the bound on $\Ex{}{X_r}$ follows from taking an expectation over the value of $L = T^{(r)}$.
\end{proof}

\section{Proof of the Lower Bound}\label{sec:lower}
We prove Theorem~\ref{thm:lower} in this section. It is sufficient to lower bound the smooth calibration error incurred by the forecaster, since by Remark~\ref{remark:optimal-transport} and Lemma~\ref{lemma:smCE-vs-lowercaldist}, we have
\[
    \caldist(x, p)
\ge \lowercaldist(x, p)
\ge \frac{1}{2}\smCE(x, p).
\]

Recall from Lemma~\ref{lemma:random-bits-polylog-smCE} that, on a random bit sequence, a forecaster \emph{might} achieve an $o(T^{1/3})$ smooth calibration error in the end. The following lemma states that, in this case, the forecaster must incur an $\Omega(T^{1/3})$ bias---defined as the difference between the total outcome and the total predictions---at some point. Later, we prove Theorem~\ref{thm:lower} by giving a simple adaptive strategy for the adversary that aims to catch this large bias.

\begin{lemma}\label{lemma:max-smCE-lower-bound}
    There exists a universal constant $c > 0$ such that the following holds for all sufficiently large $T$ and every forecaster $\A$: For $t \in [T]$, let random variable $S_t$ denote the value of $\sum_{t'=1}^{t}(x_{t'} - p_{t'})$ when $\A$ is executed against $T$ independent random bits. Then, at least one of the following two holds:
    \[
        \pr{}{\max_{t \in [T]}|S_t| \ge cT^{1/3}} \ge c
    \]
    or
    \[
        \Ex{}{\smCE(x, p)} \ge cT^{1/3},
    \]
    where the probability and expectation are over the randomness in both the random bits and the algorithm $\A$.
\end{lemma}

We first show how Theorem~\ref{thm:lower} follows from Lemma~\ref{lemma:max-smCE-lower-bound}. Given an algorithm for the forecaster, if the second condition in the lemma holds, we immediately get the desired lower bound. If the first condition holds, we let the adversary keep outputting independent random bits until $|S_t|$ reaches $cT^{1/3}$, at which point the adversary deviates from giving random bits, and starts outputting a fixed bit instead. The key is to ensure that after the adversary deviates, the smooth calibration error at the end of the $T$ steps is still $\Omega(T^{1/3})$.

\begin{proof}[Proof of Theorem~\ref{thm:lower}]
    Let $c$ be the constant in Lemma~\ref{lemma:max-smCE-lower-bound}. Consider the following \emph{mixed} adversary:
    \begin{itemize}
        \item First, the adversary decides whether it is \emph{oblivious} or \emph{adaptive} uniformly at random.
        \item If the adversary decides to be oblivious, output $T$ independent random bits; otherwise, proceed with the following steps.
        \item At each step $t$, independently draw the outcome $x_t \sim \Bern(1/2)$, until the game ends or $S_t \ge c T^{1/3}$ holds at some point.
        \item If $S_t > 0$, keep outputting bit $1$ for the remaining $T - t$ steps; otherwise, output zeros for the rest of the game.
    \end{itemize}

    Fix an arbitrary algorithm $\A$. Since $\caldist(x, p) \ge \frac{1}{2}\smCE(x, p)$, it suffices to prove that $\A$ incurs an $\Omega(T^{1/3})$ smooth calibration error against the mixed adversary defined above. By Lemma~\ref{lemma:max-smCE-lower-bound}, at least one of the two conditions in the lemma must hold when $\A$ runs on $T$ random bits. If the latter holds, i.e., $\Ex{}{\smCE(x, p)} \ge cT^{1/3}$, we get the desired lower bound. This is because the mixed adversary chooses to be oblivious with probability $1/2$, which implies a lower bound of $cT^{1/3} / 2 = \Omega(T^{1/3})$ on the expected smooth calibration error.
    
    Otherwise, assume that the former condition holds. Consider two instances of the algorithm, denoted by $\A_1$ and $\A_2$, such that $\A_1$ runs on the oblivious adversary, whereas $\A_2$ runs against the adaptive adversary. Importantly, the two instances are coupled such that the two adversaries use the same random bits, and $\A_1$ and $\A_2$ share their internal randomness.

    Let $\event$ denote the event that, for instance $\A_1$, $\max_{t \in [T]}|S_t| \ge cT^{1/3}$ holds. We will show that, whenever event $\event$ happens, the other instance $\A_2$ gives $\smCE(x, p) \ge cT^{1/3}$. It would then follow from the first condition in Lemma~\ref{lemma:max-smCE-lower-bound} that
    \[
        \pr{\A_2}{\smCE(x, p) \ge cT^{1/3}}
    \ge \pr{\A_1}{\event}
    \ge c.
    \]
    Finally, since the mixed adversary decides to be adaptive with probability $1/2$, we have a lower bound of $\frac{c}{2} \cdot cT^{1/3} = \Omega(T^{1/3})$ on the smooth calibration error when $\A$ faces the mixed adversary.

    Now assume that $\event$ happens. Let $t$ be the first time step at which $|S_t| \ge cT^{1/3}$ holds. By definition, the adaptive adversary in the $\A_2$ instance deviates from the random bits at time $t$. If $S_t > 0$, we have $x_{t+1} = x_{t+2} = \cdots = x_T = 1$, which implies
    \[
        \smCE(x, p)
    \ge \sum_{t'=1}^{T}(x_{t'} - p_{t'})
    \ge \sum_{t'=1}^{t}(x_{t'} - p_{t'})
    =   S_t \ge cT^{1/3}.
    \]
    Similarly, if $S_t < 0$, we get $x_{t+1} = \cdots = x_T = 0$, which gives
    \[
        \smCE(x, p)
    \ge \sum_{t'=1}^{T}(-1)\cdot(x_{t'} - p_{t'})
    =   \sum_{t'=1}^{T}(p_{t'} - x_{t'})
    \ge \sum_{t'=1}^{t}(p_{t'} - x_{t'})
    =   -S_t
    \ge cT^{1/3}.
    \]
    This completes the proof.
\end{proof}

To prove Lemma~\ref{lemma:max-smCE-lower-bound}, we use the following standard anti-concentration bound for Binomial distributions, for which we give a proof in Appendix~\ref{sec:lower-proofs} for completeness.

\begin{lemma}\label{lemma:binomial-anti-concentration}
    For all sufficiently large integer $n$,
    \[
        \pr{X \sim \Binomial(n, 1/2)}{|X - n / 2| \ge \sqrt{n} / 10}
    \ge \frac{3}{4}.
    \]
\end{lemma}

Now, we prove Lemma~\ref{lemma:max-smCE-lower-bound}.

\begin{proof}[Proof of Lemma~\ref{lemma:max-smCE-lower-bound}]
    We prove the lemma for $c = 1/3600$. Fix an algorithm $\A$. Let random variable $\eps_t$ denote the value of $p_t - 1/2$ when we run $\A$ on a sequence of $T$ random bits. The key quantity in our proof will be $X \coloneqq \sum_{t=1}^{T}\eps_t^2$.

    \paragraph{Case 1. $X$ is large in expectation.} If $\Ex{}{X} \ge cT^{1/3}$, we argue that $\smCE(x, p)$ will be large as well. Consider the function $f: v \mapsto 1/2 - v$ over $[0, 1]$, which is clearly $1$-Lipschitz and bounded between $-1$ and $1$. By definition of the smooth calibration error,
    \[
        \smCE(x, p)
    \ge \sum_{t=1}^{T}f(p_t)\cdot (x_t - p_t)
    =   \sum_{t=1}^{T}(1/2 - p_t)\cdot (x_t - p_t).
    \]
    For every $t \in [T]$, note that $x_t$ and $p_t$ are independent conditioning on $x_{1:(t-1)}$ and $p_{1:(t-1)}$. Thus, we have $\Ex{}{(1/2 - p_t)\cdot(x_t - 1/2)} = 0$, which implies
    \[
        \Ex{}{(1/2 - p_t)\cdot(x_t - p_t)}
    =   \Ex{}{(1/2 - p_t)\cdot (x_t - 1/2)}  - \Ex{}{(1/2 - p_t)\cdot (p_t - 1/2)}
    =   \Ex{}{\eps_t^2}.
    \]
    We conclude that $\Ex{}{\smCE(x, p)}
    \ge \Ex{}{\sum_{t=1}^{T}\eps_t^2}
    =   \Ex{}{X}
    \ge cT^{1/3}$.

    \paragraph{Case 2. $X$ is small in expectation.} Pick $L = \lfloor T^{2/3}\rfloor$ and let $m = \lfloor T / L\rfloor \ge T^{1/3}$. We divide the first $mL \le T$ steps into $m$ epochs, each of length $L$. For $i \in [m]$, let $\T_i \coloneqq \{(i - 1)\cdot L + 1, (i - 1) \cdot L + 2, \ldots, i \cdot L\}$ denote the time steps in the $i$-th epoch. We say that the $i$-th epoch is \emph{good}, if
    \[
        \left|\sum_{t \in \T_i}x_t - \frac{L}{2}\right| \ge \frac{\sqrt{L}}{10}.
    \]
    The $i$-th epoch is called \emph{weak}, if
    \[
        \sum_{t \in \T_i}\eps_t^2 \le \frac{1}{400}.
    \]

    We first note that with high probability, there are many good epochs and many weak epochs. For good epochs, the claim follows from Lemma~\ref{lemma:binomial-anti-concentration} and a Chernoff bound: The quantity $\sum_{t \in \T_i}x_t$ follows the Binomial distribution $\Binomial(L, 1/2)$. By Lemma~\ref{lemma:binomial-anti-concentration}, the probability for each epoch to be good is at least $3/4$, as long as $L$ is sufficiently large. By a Chernoff bound, with probability at least $1 - \exp(-\Omega(m))$, there are at least $\frac{2}{3}m$ good epochs. Again, when $T$ is sufficiently large, the failure probability $\exp(-\Omega(m))$ is at most $1/3$.

    For weak epochs, recall that we are under the assumption that $\Ex{}{X} = \Ex{}{\sum_{t=1}^{T}\eps_t^2} \le cT^{1/3} = T^{1/3} / 3600$. The expected number of epochs that are not weak is then at most $\frac{T^{1/3} / 3600}{1/400} = \frac{1}{9}T^{1/3}$. By Markov's inequality, the probability that there are $\ge \frac{1}{3}T^{1/3}$ epochs that are not weak is at most $\frac{\frac{1}{9}T^{1/3}}{\frac{1}{3}T^{1/3}} = \frac{1}{3}$. In other words, with probability at least $2/3$, there are at least $\frac{2}{3}T^{1/3} \ge \frac{2}{3}m$ weak epochs.

    By a union bound, with probability at least $1/3$, there are at least $\frac{2}{3}m$ good epochs and at least $\frac{2}{3}m$ weak epochs. In particular, with probability at least $1/3$, there exists an epoch $i$ that is both good and weak. Recall that $S_t$ is the partial sum of $(x_t - p_t)$. We have
    \[
        S_{iL} - S_{(i-1)L}
    =   \sum_{t \in \T_i}(x_t - p_t)
    =   \sum_{t \in \T_i}(x_t - 1/2) - \sum_{t \in \T_i}(p_t - 1/2)
    =   \left(\sum_{t \in \T_i}x_t - \frac{L}{2}\right) - \sum_{t \in \T_i}\eps_t.
    \]
    By definition of good and weak epochs, we have $\left|\sum_{t \in \T_i}x_t - \frac{L}{2}\right| \ge \frac{\sqrt{L}}{10}$ and $\left|\sum_{t \in \T_i}\eps_t\right|
    \le \sqrt{L}\cdot\sqrt{\sum_{t \in \T_i}\eps_t^2}
    \le \frac{\sqrt{L}}{20}$, where the second bound follows from the Cauchy-Schwarz inequality.
    These two bounds further imply
    \[
        |S_{iL} - S_{(i-1)L}| \ge \frac{\sqrt{L}}{10} - \frac{\sqrt{L}}{20}
        =   \frac{\sqrt{L}}{20},
    \]
    and thus,
    \[
        \max_{t \in [T]}|S_t|
    \ge \max\{S_{iL}, S_{(i-1)L}\}
    \ge \frac{\sqrt{L}}{40}.
    \]
    Recall that $L = \lfloor T^{2/3} \rfloor$ and $c = 1/3600$. For all sufficiently large $T$, the above gives a stronger guarantee than $\pr{}{\max_{t \in [T]}|S_t| \ge cT^{1/3}} \ge c$. This completes the proof.
\end{proof}

\bibliographystyle{alpha}
\bibliography{main}

\newpage

\appendix

\section{Basic Facts about Calibration Measures}\label{sec:preliminaries-proofs}

We first show that the calibration distance is always upper bounded by the ECE.

\begin{proposition}\label{prop:caldist-vs-ECE}
    For any $x \in \{0, 1\}^T$ and $p \in [0, 1]^T$,
    \[
        \caldist(x, p) \le \ECE(x, p).
    \]
\end{proposition}

\begin{proof}
    For each $t \in [T]$, define
    \[
        q_t \coloneqq g(p_t) \coloneqq \frac{\sum_{t'=1}^{T}x_{t'}\cdot\1{p_{t'} = p_t}}{\sum_{t'=1}^{T}\1{p_{t'} = p_t}}
    \]
    as the actual frequency of ones when the value $p_t$ is predicted. Note that $q_t$ is completely determined by $p_t$.

    We first show that $q$ is in $\Cal(x)$: For any $\alpha \in [0, 1]$, it holds that
    \begin{align*}
        \sum_{t=1}^{T}(x_t - q_t)\cdot\1{q_t = \alpha}
    &=  \sum_{\beta \in [0, 1]}\1{g(\beta) = \alpha}\sum_{t=1}^{T}(x_t - \alpha)\cdot\1{p_t = \beta}\\
    &=  \sum_{\beta \in [0, 1]}\1{g(\beta) = \alpha}\left[\sum_{t=1}^{T}x_t\cdot\1{p_t = \beta} - g(\beta)\cdot\sum_{t=1}^{T}\1{p_t = \beta}\right]\\
    &=  0,
    \end{align*}
    where the last step follows from the definition of $g(\cdot)$.
    This shows $q \in \Cal(x)$.

    Then, we compute the distance $\|p - q\|_1$:
    \[
        \|p - q\|_1
    =   \sum_{t=1}^{T}|p_t - q_t|
    =   \sum_{\beta \in [0, 1]}\sum_{t=1}^{T}|p_t - q_t|\cdot\1{p_t = \beta}.
    \]
    Fix $\beta \in \{p_1, p_2, \ldots, p_T\}$. We note that $p_t = \beta$ implies $q_t = g(\beta)$. It follows that
    \begin{align*}
        \sum_{t=1}^{T}|p_t - q_t|\cdot\1{p_t = \beta}
    &=  \left|\beta - g(\beta)\right|\cdot\sum_{t=1}^{T}\1{p_t = \beta}\\
    &=  \left|\beta\cdot\sum_{t=1}^{T}\1{p_t = \beta} - g(\beta)\cdot\sum_{t=1}^{T}\1{p_t = \beta}\right|\\
    &=  \left|\sum_{t=1}^{T}p_t\cdot\1{p_t = \beta} - \sum_{t=1}^{T}x_t\cdot\1{p_t = \beta}\right|\\
    &=  \left|\sum_{t=1}^{T}(x_t - p_t)\cdot\1{p_t = \beta}\right|.
    \end{align*}

    Summing over $\beta \in [0, 1]$ gives
    \[
        \|p - q\|_1
    =   \sum_{\beta \in [0, 1]}\left|\sum_{t=1}^{T}(x_t - p_t)\cdot\1{p_t = \beta}\right|
    =   \ECE(x, p).
    \]

    Since we showed that $q \in \Cal(x)$ and $\|p - q\|_1 = \ECE(x, p)$, the definition of $\caldist(x, p)$ immediately implies $\caldist(x, p) \le \ECE(x, p)$.
\end{proof}

The following lemma states the Lipschitz continuity of $\caldist$, $\lowercaldist$ and $\smCE$.

\begin{lemma}\label{lemma:Lipschitz-continuity}
    ~~The following inequalities hold for any $x \in \{0, 1\}^T$ and $p, \tilde p \in [0, 1]^T$:
    \[
        |\caldist(x, p) - \caldist(x, \tilde p)| \le \|p - \tilde p\|_1,
    \]
    \[
        |\lowercaldist(x, p) - \lowercaldist(x, \tilde p)| \le \|p - \tilde p\|_1,
    \]
    \[
        |\smCE(x, p) - \smCE(x, \tilde p)| \le 2\|p - \tilde p\|_1.
    \]
\end{lemma}

\begin{proof}
    By definition,
    \[
        \caldist(x, p)
    =   \min_{q \in \Cal(x)}\|p - q\|_1
    \le \min_{q \in \Cal(x)}\left(\|p - \tilde p\|_1 + \|\tilde p - q\|_1\right)
    =   \caldist(x, \tilde p) + \|p - \tilde p\|_1.
    \]
    By symmetry, we have $\caldist(x, \tilde p) \le \caldist(x, p) + \|p - \tilde p\|_1$. This proves the first inequality.
    
    The proof for $\lowercaldist$ is analogous:
    \begin{align*}
        \lowercaldist(x, p)
    &=  \inf_{\D \in \lowerCal(x)}\sum_{t=1}^{T}\Ex{q_t \sim \D_t}{|p_t - q_t|}\\
    &\le\inf_{\D \in \lowerCal(x)}\left[\sum_{t=1}^{T}\Ex{q_t \sim \D_t}{|p_t - \tilde p_t| + |\tilde p_t - q_t|}\right]\\
    &=  \inf_{\D \in \lowerCal(x)}\left[\sum_{t=1}^{T}\Ex{q_t \sim \D_t}{|\tilde p_t - q_t|}\right] + \|p - \tilde p\|_1\\
    &=  \lowercaldist(x, \tilde p) + \|p - \tilde p\|_1.
    \end{align*}
    Again, we also have $\lowercaldist(x, \tilde p) \le \lowercaldist(x, p) + \|p - \tilde p\|_1$. This proves the second inequality.
    
    Let $f$ be a $1$-Lipschitz function from $[0, 1]$ to $[-1, 1]$. Note that for any $p, q \in [0, 1]$ and $x \in \{0, 1\}$, we have
    \begin{align*}
        &~\left|f(p)\cdot(x - p) - f(q)\cdot(x - q)\right|\\
    \le &~\left|f(p)\cdot(x - p) - f(p)\cdot(x - q)\right| + \left|f(p)\cdot(x - q) - f(q)\cdot(x - q)\right|\\
    =   &~|f(p)|\cdot|p - q| + |f(p) - f(q)|\cdot|x - q|\\
    \le &~2|p-q|.
    \end{align*}
    The last step follows from $|f(p)| \le 1$, $|f(p) - f(q)| \le |p - q|$, and $|x - q| \le 1$.

    Thus, by definition of the smooth calibration error,
    \begin{align*}
        \smCE(x, p)
    &=  \sup_{f \in \F}\sum_{t=1}^{T}f(p_t)\cdot(x_t - p_t)\\
    &\le\sup_{f \in \F}\sum_{t=1}^{T}\left[f(\tilde p_t)\cdot(x_t - \tilde p_t) + 2|p_t - \tilde p_t|\right]\\
    &=  \sup_{f \in \F}\sum_{t=1}^{T}f(\tilde p_t)\cdot(x_t - \tilde p_t) + 2\|p - \tilde p\|_1\\
    &=  \smCE(x, \tilde p) + 2\|p - \tilde p\|_1.
    \end{align*}
    By symmetry, we also have $\smCE(x, \tilde p) \le \smCE(x, p) + 2\|p - \tilde p\|_1$. This completes the proof.
\end{proof}

\section{Proofs for Section~\ref{sec:approximation}}\label{sec:approximation-proofs}
\subsection{Failure of Na\"ive Consolidation}
In the context of Lemma~\ref{lemma:sparse-destination}, we give a concrete example in which the straightforward way of merging the transportation fails to keep the cost low. This explains why the proof of Lemma~\ref{lemma:sparse-destination} involves a complicated consolidation strategy.

Let $T = 2k$. The outcomes $x = (1, 0, \ldots, 0, 1, 1, \ldots, 1, 0)$ consist of $1$ followed by $k - 1$ copies of zeros, $k - 1$ copies of ones and a single zero. The predictions $p = (\eps, \eps, \ldots, \eps, 1 - \eps, 1 - \eps, \ldots, 1 - \eps)$ contain $k$ copies of $\eps$ followed by $k$ copies of $1 - \eps$, where $\eps = 1/(2k)$. Let $\D_1, \ldots, \D_k$ be degenerate distributions over $\{1/k\}$ and $\D_{k+1}, \ldots, \D_{2k}$ be degenerate distributions over $\{1 - 1/k\}$. Clearly, we have $\D \in \lowerCal(x)$. Furthermore, the cost of $\D$ is given by
\[
    \sum_{t=1}^{T}\Ex{q_t \sim \D_t}{|p_t - q_t|} = T \cdot \frac{1}{2k} = 1.
\]

With the notation in Section~\ref{sec:approx-sparse}, we have $s_1 = \eps$, $s_2 = 1 - \eps$, and all the transportation (specified by $\D$) are into the interval $[s_1, s_2]$. However, if we consolidate all these transportation, we would end up with a new destination of $1/2$, and the cost would surge to $\Omega(T)$.

\subsection{Technical Lemmas}
We state and prove all the technical lemmas used in Section~\ref{sec:approximation} here.
\begin{lemma}\label{lemma:two-point-smCE}
    For any $x, y \in [0, 1]$ and $\Delta_x, \Delta_y \in \R$, there exists a $1$-Lipschitz function $f:[0,1]\to[-1,1]$ such that
    \[
        f(x)\cdot\Delta_x + f(y)\cdot\Delta_y
    =   \begin{cases}
        |\Delta_x| + |\Delta_y|, & \Delta_x\Delta_y \ge 0,\\
        |\Delta_x + \Delta_y| + |x - y|\cdot\min\{|\Delta_x|, |\Delta_y|\}, & \Delta_x\Delta_y < 0.
    \end{cases}
    \]
\end{lemma}

\begin{proof}
    First, suppose that $\Delta_x\Delta_y \ge 0$, in which case there exists $s \in \{\pm 1\}$ such that $s\cdot\Delta_x = |\Delta_x|$ and $s\cdot\Delta_y = |\Delta_y|$. Then, for the constant function $f(v) \coloneqq s$, we have
    \[
        f(x)\cdot\Delta_x + f(y)\cdot\Delta_y
    =   |\Delta_x| + |\Delta_y|.
    \]

    Then, suppose that $\Delta_x\Delta_y < 0$ and $|\Delta_x| \ge |\Delta_y|$. We take $f(v) \coloneqq \sgn(\Delta_x)\cdot(1 - |v - x|)$, which is clearly $1$-Lipschiz and bounded between $-1$ and $+1$. This gives
    \begin{align*}
        f(x)\cdot\Delta_x + f(y)\cdot\Delta_y
    &=   \sgn(\Delta_x)\cdot \Delta_x + \sgn(\Delta_x)\cdot(1 - |x - y|)\cdot\Delta_y\\
    &=  |\Delta_x| - (1 - |x - y|)\cdot|\Delta_y|\\
    &=  (|\Delta_x| - |\Delta_y|) + |x - y|\cdot|\Delta_y|\\
    &=  |\Delta_x + \Delta_y| + |x - y|\cdot\min\{|\Delta_x|, |\Delta_y|\}.
    \end{align*}
    Similarly, when $\Delta_x\Delta_y < 0$ and $|\Delta_x| < |\Delta_y|$, taking $f(v) \coloneqq \sgn(\Delta_y)\cdot(1 - |v - y|)$ also gives
    \begin{align*}
        f(x)\cdot\Delta_x + f(y)\cdot\Delta_y
    &=   \sgn(\Delta_y)\cdot \Delta_y + \sgn(\Delta_y)\cdot(1 - |x - y|)\cdot\Delta_x\\
    &=  |\Delta_y| - (1 - |x - y|)\cdot|\Delta_x|\\
    &=  (|\Delta_y| - |\Delta_x|) + |x - y|\cdot|\Delta_x|\\
    &=  |\Delta_x + \Delta_y| + |x - y|\cdot\min\{|\Delta_x|, |\Delta_y|\}.
    \end{align*}
\end{proof}

\begin{lemma}\label{lemma:technical-ineq-1}
    For any $0 \le \alpha \le \beta \le 1$ and $p \in [0, 1]$, we have
    \[
        p\cdot|p - \alpha| + (1 - p)\cdot|p - \beta|
    \le 2 \cdot \left[\left|p\cdot(1-\alpha) - (1-p)\cdot \beta\right| + (\beta - \alpha)\cdot\min\{p\cdot(1-\alpha), (1-p)\cdot \beta\}\right].
    \]
\end{lemma}

\begin{proof}
    We prove the inequality for the following three cases separately.
    
    \paragraph{Case 1: $p \le \alpha$.} In this case, the left-hand side of the inequality gets reduced to
        \[
            p\cdot(\alpha - p) + (1 - p)\cdot(\beta - p)
        =   (1-p)\cdot\beta - p \cdot (1 - \alpha)
        \le |p\cdot(1 - \alpha) - (1 - p)\cdot\beta|,
        \]
        which is clearly upper bounded by the right-hand side.
    
    \paragraph{Case 2: $p \ge \beta$.} Similarly, we can simplify the left-hand side to
        \[
            p\cdot(p - \alpha) + (1 - p)\cdot(p - \beta)
        =   p\cdot(1 - \alpha) - (1 - p)\cdot\beta
        \le |p\cdot(1 - \alpha) - (1 - p)\cdot\beta|,
        \]
        which is, again, upper bounded by the right-hand side.
    
    \paragraph{Case 3: $p \in (\alpha, \beta)$.} In this case, using the identities $x + y = |x - y| + 2\min\{x, y\}$ and 
        \[
            p\cdot(1-\alpha) - p\cdot(p - \alpha) = (1 - p)\cdot\beta - (1 - p)\cdot(\beta - p) = p(1-p),
        \]
        we can write the left-hand side as:
        \begin{align*}
            &~|p\cdot(p - \alpha) - (1 - p)\cdot(\beta - p)| + 2\cdot\min\{p\cdot(p - \alpha), (1 - p)\cdot(\beta - p)\}\\
        =   &~|p\cdot(1 - \alpha) - (1 - p)\cdot\beta| + 2\cdot\min\{p\cdot(p - \alpha), (1 - p)\cdot(\beta - p)\}.
        \end{align*}
        Thus, to prove the lemma, it remains to show that:
        \begin{equation}\label{eq:technical-ineq-1-case-3}
            \min\{p\cdot(p - \alpha), (1 - p)\cdot(\beta - p)\}
        \le (\beta - \alpha)\cdot\min\{p\cdot(1 - \alpha), (1 - p)\cdot \beta\}.
        \end{equation}

        If $p\cdot(p-\alpha) \le (1 - p)\cdot(\beta - p)$, the minimum on the right-hand side of \eqref{eq:technical-ineq-1-case-3} is also achieved by the first term. Then, it is sufficient to prove that $p\cdot(p - \alpha)
        \le (\beta - \alpha)\cdot p\cdot(1 - \alpha)$,
        which is equivalent to
        \[
            p \le \beta - \alpha\beta + \alpha^2.
        \]
        Note that the assumption $p\cdot (p - \alpha) \le (1 - p)\cdot(\beta - p)$ is equivalent to
        \[
            p + p\cdot(\beta-\alpha) \le \beta.
        \]
        Since $\beta - \alpha \ge 0$ and $p \ge \alpha$, we have
        \[
            \alpha\cdot(\beta - \alpha)
        \le p\cdot(\beta - \alpha).
        \]
        Adding the two inequalities above together gives the desired inequality $p \le \beta - \alpha\beta + \alpha^2$.

        The remaining case that $p\cdot(p-\alpha) > (1 - p)\cdot(\beta - p)$ can be dealt with in a similar way. In this case, Inequality~\eqref{eq:technical-ineq-1-case-3} is equivalent to
        \[
            \beta - p \le \beta^2 - \alpha\beta.
        \]
        The assumption $p\cdot(p-\alpha) > (1 - p)\cdot(\beta - p)$ implies
        \[
            \beta - p
        \le p\cdot(\beta - \alpha).
        \]
        Applying $\beta - \alpha \ge 0$ and $p \le \beta$ to the above, we get
        \[
            \beta - p \le \beta^2 - \alpha\beta
        \]
        as desired.
\end{proof}

\begin{lemma}\label{lemma:technical-ineq-2}
    For any $0 \le \alpha \le \beta \le 1$ and $p \in [0, 1]$, we have
    \begin{align*}
        &~\min\left\{(1 - p)\cdot|p - \alpha| + p\cdot|p - \beta|, (1 - p) \cdot \alpha + p \cdot (1 - \beta)\right\}\\
    \le &~10 \cdot \left[\left|(1 - p)\cdot \alpha - p\cdot (1 - \beta)\right| + (\beta - \alpha)\cdot\min\{(1 - p)\cdot \alpha, p\cdot (1 - \beta)\}\right].
    \end{align*}
\end{lemma}

\begin{proof}
    As in the proof of Lemma~\ref{lemma:technical-ineq-1}, we consider the following three cases.

    \paragraph{Case 1: $p \le \alpha$.} The left-hand side is upper bounded by the first term in the minimum, which, in this case, is given by
    \[
        (1-p)\cdot(\alpha-p) + p\cdot(\beta - p)
    =   (1-p)\cdot\alpha - p\cdot(1 - \beta)
    \le |(1-p)\cdot\alpha - p\cdot(1 - \beta)|.
    \]
    Clearly, this is upper bounded by the right-hand side.

    \paragraph{Case 2: $p \ge \beta$.} Similarly, we can upper bound the left-hand side by
    \[
        (1-p)\cdot(p-\alpha) + p\cdot(p - \beta)
    =   -(1-p)\cdot\alpha + p\cdot(1 - \beta)
    \le |(1-p)\cdot\alpha - p\cdot(1 - \beta)|,
    \]
    which, in turn, is at most the right-hand side.

    \paragraph{Case 3: $p \in (\alpha, \beta)$.} We first consider the case that $\beta - \alpha$ is large. Concretely, suppose that $\beta - \alpha \ge 1/5$. If so, by the identity $x + y = |x-y| + 2\min\{x,y\}$, we have
    \begin{align*}
        (1 - p)\cdot\alpha + p \cdot(1-\beta)
    &=  |(1 - p)\cdot\alpha - p \cdot(1-\beta)| + 2\min\{(1 - p)\cdot\alpha, p \cdot(1-\beta)\}\\
    &\le10|(1 - p)\cdot\alpha - p \cdot(1-\beta)| + 10(\beta - \alpha)\min\{(1 - p)\cdot\alpha, p \cdot(1-\beta)\},
    \end{align*}
    which implies the desired inequality.

    We then focus on the case that $\beta - \alpha < 1/5$. We write $x \coloneqq p - \alpha > 0$ and $y \coloneqq \beta - p > 0$. Note that
    \[
        (1-p)\cdot\alpha - p\cdot(1 - \beta)
    =   p\cdot(\beta - p) - (1 - p)\cdot(p - \alpha)
    =   py - (1-p)x.
    \]
    We claim that if $(1-p)x$ and $py$ are not close (up to a multiplicative factor), we are done. Formally, suppose that
    \[
        \frac{\min\{(1-p)x, py\}}{\max\{(1-p)x, py\}} \le \frac{2}{3}.
    \]
    Then, we may upper bound the left-hand side of the desired inequality by
    \[
        (1-p)x + py
    =   \max\{(1-p)x, py\} + \min\{(1-p)x, py\}
    \le \frac{5}{3}\max\{(1-p)x, py\}.
    \]
    On the other hand, the right-hand side is lower bounded by its first term, namely,
    \begin{align*}
        10|(1-p)\cdot \alpha - p\cdot(1 - \beta)|
    &=  10|py - (1-p)x|\\
    &=  10\max\{(1-p)x, py\} - 10\min\{(1-p)x, py\}\\
    &\ge\left(10 - 10\cdot\frac{2}{3}\right)\cdot \max\{(1-p)x, py\}\\
    &\ge\frac{5}{3}\cdot \max\{(1-p)x, py\}.
    \end{align*}
    This proves the inequality when $\min\{(1-p)x, py\} \le \frac{2}{3}\max\{(1-p)x, py\}$ holds.

    Finally, we deal with the case that both $\beta - \alpha < 1/5$ and $\min\{(1-p)x, py\} > \frac{2}{3}\max\{(1-p)x, py\}$ hold. Note that the second condition implies $\frac{(1-p)x}{py} > \frac{2}{3}$ and $\frac{py}{(1-p)x} > \frac{2}{3}$. Again, we simplify and relax the desired inequality into
    \begin{align*}
        (1-p)x + py
    &\le10(x+y)\min\{(1-p)\cdot\alpha, p\cdot (1 - \beta)\}\\
    &=  10(x+y)\cdot[p(1-p) - \max\{(1-p)x, py\}].
    \end{align*}
    We argue that both $(1-p)x$ and $py$ are at most $\frac{3}{10}p(1-p)$. Otherwise, suppose that $(1-p)x > \frac{3}{10}p(1-p)$. This implies $x > \frac{3}{10}p > p/5$. Furthermore, we have
    \[
        py
    >   \frac{2}{3}(1-p)x
    >   \frac{2}{3}\cdot\frac{3}{10}p(1-p),
    \]
    which implies $y > (1-p)/5$. We then obtain $x + y > p/5 + (1-p)/5 = 1/5$, which contradicts $x + y = \beta - \alpha < 1/5$. An analogous argument also rules out the possibility that $py > \frac{3}{10}p(1-p)$.

    Therefore, it suffices to prove that
    \[
        (1-p)x + py
    \le 10(x+y)\cdot\left[p(1-p) - \frac{3}{10}p(1-p)\right]
    =   7(x+y)\cdot p(1-p),
    \]
    or, equivalently,
    \[
        (1-p)\cdot\frac{x}{x+y} + p\cdot\frac{y}{x+y}
    \le 7p(1-p).
    \]
    The first term on the left-hand side above can be upper bounded as follows:
    \[
        (1-p)\cdot\frac{x}{x+y}
    =   (1-p)\cdot\frac{1}{1 + y/x}
    \le (1-p)\cdot\frac{1}{1 + \frac{2\cdot(1-p)}{3p}}
    =   \frac{p(1-p)}{\frac{2}{3} + \frac{p}{3}}
    \le \frac{3}{2}\cdot p(1-p),
    \]
    where the second step applies $y/x > \frac{2\cdot(1-p)}{3p}$, which follows from $\frac{py}{(1-p)x} > \frac{2}{3}$.
    Similarly, we have
    \[
        p\cdot\frac{y}{x+y}
    =   p\cdot\frac{1}{x/y+1}
    \le p\cdot\frac{1}{\frac{2p}{3\cdot(1-p)} + 1}
    =   \frac{p(1-p)}{1-\frac{1}{3}p}
    \le \frac{3}{2}\cdot p(1-p).
    \]
    Adding the two inequalities above gives
    \[
        (1-p)\cdot\frac{x}{x+y} + p\cdot\frac{y}{x+y} \le 3p(1-p) \le 7p\cdot(1-p),
    \]
    which implies the desired inequality for the last case, and thus completes the proof.
\end{proof}

\section{Proof for Section~\ref{sec:lower}}\label{sec:lower-proofs}
We prove Lemma~\ref{lemma:binomial-anti-concentration}, which is restated below.

\vspace{6pt}

\noindent \textbf{Lemma~\ref{lemma:binomial-anti-concentration}.}{~\it
    For all sufficiently large integer $n$,
    \[
        \pr{X \sim \Binomial(n, 1/2)}{|X - n / 2| \ge \sqrt{n} / 10}
    \ge \frac{3}{4}.
    \]
}

\begin{proof}
    The mode of $\Binomial(n, 1/2)$ is $\lfloor n / 2\rfloor$. When $n = 2k$ is even, it holds for every $j \in \{0, 1, \ldots, n\}$ that
    \begin{align*}
        \pr{X \sim \Binomial(2k, 1/2)}{X = j}
    &\le\pr{X \sim \Binomial(2k, 1/2)}{X = k}\\
    &=  2^{-2k}\frac{(2k)!}{(k!)^2}\\
    &=  (1 + o_n(1))\cdot 2^{-2k}\cdot \frac{\sqrt{2\pi\cdot 2k}\cdot(2k/e)^{2k}}{2\pi k\cdot (k/e)^{2k}}\\
    &=  (1 + o_n(1))\cdot \frac{\sqrt{2/\pi}}{\sqrt{n}}.
    \end{align*}
    The third step applies Stirling's approximation $\frac{n!}{\sqrt{2\pi n}(n/e)^n} = 1 + o_n(1)$. Since $\sqrt{2/\pi} < 1$, for sufficiently large $n$ we have an upper bound of $1/\sqrt{n}$.
    Similarly, when $n = 2k + 1$ is odd, we have
    \begin{align*}
        \pr{X \sim \Binomial(n, 1/2)}{X = j}
    &\le2^{-(2k+1)}\cdot\frac{(2k+1)!}{k!(k+1)!}\\
    &=  (1 + o_n(1))\cdot 2^{-(2k+1)}\cdot\frac{\sqrt{2\pi(2k+1)}\cdot\left(\frac{2k+1}{e}\right)^{2k+1}}{\sqrt{2\pi k}\cdot\sqrt{2\pi(k+1)}\cdot (k/e)^k\left(\frac{k+1}{e}\right)^{k+1}}\\
    &=  (1 + o_n(1))\cdot \frac{1}{\sqrt{\pi k}}\cdot\left(1 + \frac{1}{2k}\right)^k\cdot\left(1 - \frac{1}{2(k+1)}\right)^{k+1}\\
    &=  (1 + o_n(1))\cdot \frac{1}{\sqrt{\pi k}}\cdot(e^{1/2} + o_n(1))\cdot(e^{-1/2} + o_n(1))\\
    &=  (1 + o_n(1))\cdot \frac{\sqrt{2/\pi}}{\sqrt{n}}.
    \end{align*}

    Therefore,
    \begin{align*}
        &~\pr{X \sim \Binomial(n, 1/2)}{|X - n / 2| < \sqrt{n} / 10}\\
    =   &~\sum_{j=0}^{n}\pr{X \sim \Binomial(n, 1/2)}{X = j}\cdot \1{|j - n/2| < \sqrt{n} / 10}\\
    \le &~\frac{1}{\sqrt{n}}\cdot \left(2\cdot\frac{\sqrt{n}}{10} + 1\right)
    \le \frac{1}{4},
    \end{align*}
    where the last step holds for all sufficiently large $n$.
\end{proof}

\end{document}